\documentclass[12pt]{alt2020}
\pdfoutput=1

\title[Sampling Without Compromising Accuracy in Adaptive Data Analysis]{Sampling Without Compromising Accuracy in Adaptive Data Analysis}
\usepackage{times}

\altauthor{\Name{Benjamin Fish} \Email{benjamin.fish@microsoft.com} \\
	\addr Microsoft Research \\
	Montr\'{e}al, H3A 3H3, Canada
	\AND 
	\Name{Lev Reyzin} \Email{lreyzin@uic.edu} \\
	\addr Department of Mathematics, Statistics, and Computer Science \\
	University of Illinois at Chicago  \\
	Chicago, IL 60607%-7045
	, USA
	\AND 
	\Name{Benjamin I. P. Rubinstein} \Email{brubinstein@unimelb.edu.au} \\
	\addr School of Computing \& Information Systems \\
	University of Melbourne \\
	Parkville, 3010 VIC, Australia}

\usepackage{microtype}
\usepackage{graphicx}
\usepackage{booktabs} % for professional tables

\usepackage{hyperref}

\usepackage{url}            % simple URL typesetting
\usepackage{amsfonts}       % blackboard math symbols
\usepackage{nicefrac}       % compact symbols for 1/2, etc.

\usepackage{lastpage} % only for LastPage otherwise delete

\usepackage{natbib}
\usepackage{mathtools}
\usepackage{algorithmic, algorithm}
\usepackage{makecell}
\usepackage{amsmath, amssymb} %, amsthm} proofs in jmlr style...
\usepackage{xr}

\DeclareMathOperator*{\argmin}{arg\,min}
\newcommand{\M}{\mathcal{M}} 
\newcommand{\A}{\mathcal{A}} 
\renewcommand{\Pr}{\mathbb{P}}
\newcommand{\E}{\mathbb{E}}
\newcommand{\ls}{\mathcal{L}}

\newenvironment{proofof}[1]{\par\noindent{\bf Proof of #1\ }}{\hfill\BlackBox\\[2mm]}

\begin{document}
\maketitle

\begin{abstract}%   <- trailing '%' for backward compatibility of .sty file
%We initiate the study of sublinear-time mechanisms to answer adaptive queries into datasets.  
In this work, we study how to use sampling to speed up mechanisms for answering adaptive queries into datasets without reducing the accuracy of those mechanisms.  This is important to do when both the datasets and the number of queries asked are very large.  In particular, we describe a mechanism that provides a polynomial speed-up per query over previous mechanisms, without needing to increase the total amount of data required to maintain the same generalization error as before.  We prove that this speed-up holds for arbitrary statistical queries.  We also provide an even faster method for achieving statistically-meaningful responses wherein the mechanism is only allowed to see a constant number of samples from the data per query.  Finally, we show that our general results yield a simple, fast, and unified approach for adaptively optimizing convex and strongly convex functions over a dataset.
%
%Our general results yield improved bounds for adaptively optimizing convex and strongly convex functions over a dataset; in particular, we improve the dependence on the number of queries the sample complexity of convex optimization.
\end{abstract}

\begin{keywords}
Adaptive data analysis, differential privacy, sublinear-time algorithms
\end{keywords}

\section{Introduction}

The field of data analysis seeks out statistically valid conclusions from data: inferences that generalize to an underlying distribution rather than specialize to the data sample at hand.  As a result, classical proofs of statistical efficiency have focused on independence assumptions on data with a pre-determined sequence of analyses~\citep{lee2016}.  In practice, most data analysis is adaptive or exploratory:  previous inferences inform future analysis.  This adaptivity is nigh impossible to avoid when multiple scientists contribute work to an area of study using the same or similar data sets.  Unfortunately, adaptivity may lead to `false discovery,' where the dependence on past analysis may create pervasive overfitting---also known as `the garden of forking paths' or `$p$ hacking'~\citep{gelman2014statistical}. 
%While basing each analysis on new data drawn from the same distribution might appear an appealing solution, repeated data collection and analysis time can be prohibitively costly.

%The field of data analysis, roughly, tries to discover information from a
%database that is representative of a ground truth distribution.  To avoid
%{false discovery} one has to limit the number of queries to the database.
%In \emph{adaptive} data analysis, we are concerned with preserving the generalizability
%of our queries when the queries are adaptively selected -- this is much harder to do
%in the adaptive setting where VC bounds and other canonical statistical methods such as
%the Bonferroni correction need to be avoided for the results to remain meaningful.

There has been much recent progress in minimizing the amount of data needed to draw generalizable conclusions, without having to make any assumptions about the type of adaptations used by the data analysis.
%While there has been much recent progress in the field of adaptive data analysis,
%, two important facets of this area have been largely overlooked.  The first is that many of the approaches that are widely used in practice have not been theoretically studied.  
%The second, and related point, is that 
Meanwhile, bootstrapping and related sampling techniques have enjoyed widespread and successful use in practice across a variety of settings~\citep{kleiner2012bootstrap,xiao2016confidence}, including in adaptive settings~\citep{golbandi2011adaptive},
but they have been largely ignored in this burgeoning field.
%This practice includes much in the literature on sublinear-time algorithms, large-scale learning, and others---far too numerous to detail here.
This is a gap that not only points to an unexplored area of theoretical study, but also
opens up the possibility of creating substantially faster algorithms for answering 
adaptively generated queries.

In this paper, we aim to do just this: we develop strong theoretical results that are significantly faster than previous approaches, taking sublinear time per query, thereby initiating the intersection of sublinear-time algorithm design and adaptive data analysis.  

\begin{table*}[t]

\centering

\resizebox{\textwidth}{!}{%
\begin{tabular}{p{3cm}cccccc}
\toprule
	& \multicolumn{3}{c}{computational complexity} & \multicolumn{3}{c}{sample complexity} \\%& \multicolumn{2}{c}{iterations per query} \\
	\cmidrule{2-3}\cmidrule{4-5}\cmidrule{6-7}
	query type & previous work & \multicolumn{2}{c}{this work} & \multicolumn{2}{c}{previous work} & this work \\%& previous work & this work \\
\midrule
%  \multirow{2}{*}{Counting queries} & Sample complexity & $\tilde{O}\left(\frac{\sqrt{k}}{\alpha^2}\right)$ & $\tilde{O}\left(\frac{\sqrt{k}}{\alpha^2}\right)$\\
% & Time per query & $\tilde{O}\left(\frac{\sqrt{k}}{\alpha^2}\right)$ & $\tilde{O}\left(\frac{\log k}{\alpha^2}\right)$\\\hline
\small{statistical queries   (Section~\ref{sec:sampling})} &  $\tilde{O}\left(\frac{k^{3/2}}{\alpha^2}\right)$ & $\tilde{O}\left(\frac{k\log^2(k)}{\alpha^2}\right)$  & $\Omega\left(\frac{k}{\alpha^2}\right)$
& $\tilde{O}\left(\frac{\sqrt{k}}{\alpha^2}\right)$ & $\Omega\left(\frac{\sqrt{k}}{\alpha^2}\right)$ & $\tilde{O}\left(\frac{\sqrt{k}}{\alpha^2}\right)$ \\
%& --- & --- \\
 \midrule
\small{sampling counting queries (Section~\ref{sec:scqs})}  & --- & $\tilde{O}\left(k\log\left(\frac{k}{\alpha}\right)\right)$ & $\Omega(k)$
& --- & --- & $\tilde{O}\left(\frac{\sqrt{k}}{\alpha^2}\right)$\\
%\cmidrule{2-4}
 %& --- & --- \\
 %\midrule \midrule
%\small{convex optimization (Section~\ref{sec:optimization})}
% & $\tilde{O}\left(\frac{\sqrt{dk}}{\alpha^{2}}\right)$ & $\tilde{O}\left(\frac{d^{3/2}\sqrt{k}}{\alpha^5}\right)$ 
% & $\tilde{O}\left(\frac{dk}{\alpha^4}\right)$ & $\tilde{O}\left(\frac{d^2\log(k)}{\alpha^5}\right)$
%& $\tilde{O}\left(\frac{dk}{\alpha^4}\right)$ & $\tilde{O}\left(\frac{\log(k)}{\alpha^2}\right)$ \\
% \midrule
%\small{strongly convex \linebreak[4] optimization (Section~\ref{sec:optimization})}
% & $\tilde{O}\left(\frac{\sqrt{dk}}{\alpha^{3/2}}\right)$ & $\tilde{O}\left(\frac{d^{3/2}\sqrt{k}}{\alpha^{5/2}}\right)$
% & $\tilde{O}\left(\frac{dk}{\alpha^3}\right)$ & $\tilde{O}\left(\frac{d^2\log(k)}{\alpha^3}\right)$
%&  $\tilde{O}\left(\frac{dk}{\alpha^3}\right)$ & $\tilde{O}\left(\frac{\log(k)}{\alpha}\right)$\\
 \bottomrule
\end{tabular}
} % end resizebox

	\caption{Summary of our upper and lower bounds compared to previous work~\citep{bassily2016algorithmic} over the course of answering $k$ queries.  $\alpha$ is the accuracy rate. %and $d$ is the dimension of the data.  
	Dependence on the probability of failure has been suppressed for ease of reading.  
%Above the double line are our more general results and below are their applications to convex optimization.  Note that this table does not show the slightly different assumptions made in previous work versus this work for convex optimization.  
For more precise definitions, see Section~\ref{sec:prelim}.}  %Note here and throughout we use the notation $\tilde{O}(f)$ to hide terms that are logarithmic in $f$.}

\label{table:results1}

%\vspace{0.5cm}
%
%\begin{tabular}{p{3cm}cc}
%\toprule
%& \multicolumn{2}{c}{iterations per query}\\
%\cmidrule{2-3}
%query type & \makecell{previous work\\ \citep{bassily2016algorithmic}} & this work\\
%\midrule
%%convex\ \ \ \ \ \ \ \ \ \ \ \ \ \ \ \ \ optimization (Section~\ref{sec:optimization})
%\parbox{3cm}{\small{convex optimization (Section~\ref{sec:optimization})}}
% & $\tilde{O}\left(\frac{dk}{\alpha^4}\right)$ & $\tilde{O}\left(\frac{\log(k)}{\alpha^2}\right)$\\
% \midrule
%\parbox{3cm}{\small{strongly~convex optimization (Section~\ref{sec:optimization})}}
% & $\tilde{O}\left(\frac{dk}{\alpha^3}\right)$ & $\tilde{O}\left(\frac{\log(k)}{\alpha}\right)$\\
% \bottomrule
%\end{tabular}

\end{table*}

%\begin{table}[h]
%\begin{center}
%\begin{tabular}{p{3cm}cc}
%\toprule
%& \multicolumn{2}{c}{iterations per query}\\
%\cmidrule{2-3}
%query type & \makecell{previous work\\ \citep{bassily2016algorithmic}} & this work\\
%\midrule
%%convex\ \ \ \ \ \ \ \ \ \ \ \ \ \ \ \ \ optimization (Section~\ref{sec:optimization})
%\parbox{3cm}{\small{convex optimization (Section~\ref{sec:optimization})}}
% & $\tilde{O}\left(\frac{dk}{\alpha^4}\right)$ & $\tilde{O}\left(\frac{\log(k)}{\alpha^2}\right)$\\
% \midrule
%\parbox{3cm}{\small{strongly~convex optimization (Section~\ref{sec:optimization})}}
% & $\tilde{O}\left(\frac{dk}{\alpha^3}\right)$ & $\tilde{O}\left(\frac{\log(k)}{\alpha}\right)$\\
% \bottomrule
%\end{tabular}
%%\vspace{0.5cm}
%\end{center}
%\caption{\small Summary of our results, continued.  $k$ is the number of queries and $\alpha$ is the accuracy rate.  Dependence on the probability of failure has been suppressed for ease of reading.  Note that this table does not show the slightly different assumptions used in previous work versus this work for convex optimization.  For more precise definitions, see Section~\ref{sec:prelim}.}
%\label{table:results2}
%
%\end{table}

\subsection{Motivation and Results}
As in previous literature (starting with~\citet{dwork2015preserving}), a mechanism $\M$ is given an i.i.d.\ sample $S$ of size $n$ from an unknown distribution $D$ over a space $X$, and is supplied queries of the form $q:D\rightarrow\mathbb{R}$.  After each query, the mechanism must respond with an answer $a$ that is close to $q(D)$ up to a parameter $\alpha$ with high probability.  Furthermore, each query may be adaptive:  The query may depend on the previous queries and answers to those queries.

Our results are summarized in Table~\ref{table:results1}.  % and~\ref{table:results2}.
We first point out that we can use well-known privacy amplification techniques to get a fast mechanism for answering statistical queries (which asks questions of the form `What is the expected value of my function on the data?') \emph{without} losing accuracy.  Usually, using these privacy amplification techniques results in a loss of accuracy, so it's notable that in this setting, we can speed up responses without the loss in accuracy on the distribution.
In Section~\ref{sec:sampling}, we show that our method still has $n=\tilde{O}(\sqrt{k}/\alpha^2)$\footnote{We use the notation $\tilde{O}(f)$ to hide terms that are logarithmic in $f$.} sample complexity as in previous work but takes only $\tilde{O}(k\log^2({k})/\alpha^2)$ time to answer $k$ queries, instead of $\tilde{O}(k^{3/2}/\alpha^2)$ time as in previous approaches (Theorem~\ref{thm:lsq}).  Moreover, our mechanism to answer a query is simple, and involves subsampling $\ell=\tilde{O}(\log(k)/\alpha^2)$ samples per query.   While it is not possible to improve the sample complexity over previous work~\citep{ullman2018limits}, we decrease the number of samples that need to be examined per query, resulting in faster responses to queries.  

We also show that our upper bound on total computational complexity is tight up to poly-log factors when the mechanism gets to ask for evaluations of queries at given sample points.  This lower bound on computational complexity that we provide is larger than the sample complexity $\tilde{O}(\sqrt{k}/\alpha^2)$ for answering statistical queries.  Running time then may become a problem for very large and popular datasets, making it valuable to give provably accurate mechanisms that are fast enough to run on very large datasets when the number of queries is large compared to the size of the dataset.

However, an analyst may wish to control the number of samples $\ell$ examined to compute the response to a query, down to possibly one point, in order to save on time and effort.  %For example, the analyst might seek the answer to a counting query using a very small number of sample points from the database, even just a single sample point.  
The above methods cannot handle this case gracefully because when $\ell$ is sufficiently small, the guarantees on accuracy %(using Definition~\ref{def:lsq_accuracy} below) 
become trivial---we get only that $\alpha=O(1)$.  Instead, we want to have a statistically-meaningful reply even if $\ell=1$.  Indeed, the empirical answer when $\ell=1$ is $\{0,1\}$-valued, unlike a response using Laplacian noise.

To address these issues, we consider an `honest' setting where the mechanism must always yield a plausible reply to each query (Section~\ref{sec:scqs}).  This is analogous to the honest version~\citep{yang2001honestsq}
of the statistical query (SQ) setting for learning~\citep{BlumFJKMR94,Kearns98}, or the $1$-STAT oracle for optimization~\citep{feldman2017statistical}.  Thus we introduce \emph{sampling counting queries}, which imitate the process of an analyst requesting the value of a query on a single random sample.  Equivalently, this enforces a binary randomized response, where a query asks for a coin flip from a coin with unknown bias determined by the query on the dataset.  

This explores a different extreme than statistical queries by allowing for queries to be answered much faster than statistical queries can, at the cost of accuracy.  So, for example, we can't just round the values of statistical queries in order to answer sampling counting queries.
%Namely, we require that for a query of the form $q:X\rightarrow\{0,1\}$, the mechanism must output a $\{0,1\}$-valued answer that is accurate in expectation.  
We show how to answer these queries by sampling a single point $s$ from $S$ and then applying a simple differentially private algorithm to $q(s)$ (Theorem~\ref{thm:sampling_thm}).  
%This mechanism allows us to answer such queries faster than we can standard counting queries, while still providing useful information.
%In Appendix~\ref{sec:comparison}, we compare sampling counting queries to counting queries.

Finally, to demonstrate the applicability of our general results, we use them as a black-box technique to obtain bounds for convex optimization (Section~\ref{sec:optimization}).  In particular, we introduce a simple procedure for adaptive gradient
descent that uses our sampling mechanism for statistical queries to compute gradients in the course of gradient descent.  This results in a fast, unified approach for answering both convex and strongly convex optimization queries. For answering $k$ 
convex optimization queries, we decrease the total number of calls to compute the gradient from $O(k d n^2)$ in \citep{bassily2016algorithmic} to $\tilde{O}(kd/\alpha^2)$ in the convex case and $\tilde{O}(kd/\alpha)$ in the strongly convex case, where $d$ is the dimension of the convex space (Corollaries~\ref{thm:conv_opt_whp} and~\ref{thm:str_conv_opt_whp}).
%\citet{bassily2016algorithmic} to $O(\log k)$ in this work (Theorems~\ref{thm:conv_opt_whp} and~\ref{thm:str_conv_opt_whp}).  
%That is, while the overall sample complexity does not improve over previous work, the number of samples that need to be examined for each query does.  
%We also similarly decrease the number of iterations of gradient descent per query.  
(Note, however, \citealt{bassily2016algorithmic} make slightly different assumptions about the loss function.  Roughly speaking, they require that the loss function be bounded, whereas we only require the gradient of the loss function be bounded.)  Our results are similar to those given by~\citet{feldman2017sco} when using our statistical query mechanism to compute gradients.  However, we provide a unified approach and a direct proof using primal gradient descent, unlike~\citet{feldman2017sco}, who uses the more complex dual gradient method of~\citet{devolder2013firstorder} in the strongly convex case.

%While a nontrivial advance on its own, this contribution also highlights the applicability of our methods.

\subsection{Previous Work}
Previous work in this area has focused on finding accurate mechanisms with low sample complexity (the size of $S$) for a variety of queries and settings~\citep{bassily2016algorithmic,dwork2015generalization,dwork2015preserving,rogers2016maxinfo,steinke2015between}.   \citet{bassily2016algorithmic} consider, amongst other queries, \emph{statistical queries};
%, which are merely any function of $X^n$ whose output does not change much when the input is perturbed. 
%(for a more precise definition, see below).  
if the queries are nonadaptive, then only roughly $\log(k)/\alpha^2$ samples are needed to answer $k$ such queries.  And if the queries are adaptive but the mechanism simply outputs the empirical estimate of $q$ on $S$, then the sample complexity is much worse---order $k/\alpha^2$ instead.

In this paper, we will focus only on computationally-efficient mechanisms.  It is not necessarily obvious that it is possible to achieve a smaller sample complexity for an efficient mechanism in the adaptive case, but \citet{bassily2016algorithmic}, building on the work of \citet{dwork2015preserving}, provide a mechanism with sample complexity $n=\tilde{O}(\sqrt{k}/\alpha^2)$ to answer $k$ statistical queries.  Furthermore, for efficient mechanisms, this bound is tight in $k$~\citep{steinke2015interactive}.  \citet{bassily2016algorithmic} also show how to efficiently answer \emph{convex optimization queries}, which ask for the minimizer of a convex loss function, using a (private) gradient descent algorithm of~\cite{bassily2014private}.

This literature shows that the key to finding such mechanisms with this improvement over the na\"ive method is finding stable mechanisms:  those whose output does not change too much when the sample is changed by a single element.  Much of this literature leverages differential privacy~\citep{bassily2016algorithmic,dwork2015generalization,dwork2015preserving,steinke2015between}, which offers a strong notion of stability.
%Differentially private mechanisms can significantly improve over the mechanisms that simply output the empirical estimate of $q$ on $S$.  If the mechanism just uses this na\"ive approach, the sample complexity is order $k/\alpha^2$ instead, quadratically worse.
Here we use differentially-private mechanisms post sampling,
%as we are acutely interested in the impact on privacy when sampling.  In both theory and practice, 
noting that sampling in settings where privacy matters has long been deemed useful~\citep{bassily2014private,jorgensen2015conservative,kasiviswanathan2008learn,kellaris2013practical}.

In particular, we take advantage of the fact that sampling not only maintains privacy, but actually boosts it.  Such a result may be found in \citep{kasiviswanathan2008learn}, and since then various sampling regimes have been considered, including by \citet{bun2015subsampling}, who show that sampling with replacement boosts privacy, and more recently by \citet{balle2018privacy}, who establish tight bounds.

%In order to sample without losing accuracy on statistical queries, we take advantage of the fact that sampling not only maintains privacy, but actually boosts it.  In particular, for an $\epsilon$-private mechanism on a database of size $n$, we will require that when sampling $\ell$ points uniformly and efficiently from those $n$ points, and then applying the same mechanism, the result is $O\left(\frac{\ell}{n}\epsilon\right)$-private.  Such a result may be found in \citep{kasiviswanathan2008learn}, and since then various sampling regimes have been considered, including by \citet{bun2015subsampling}, who show that sampling with replacement boosts privacy, and more recently by \citet{balle2018privacy}, who establish tight bounds.

\section{Model and Preliminaries}\label{sec:prelim}

In the adaptive data analysis setting we consider, a (possibly stateful) mechanism $\M$ that is given an i.i.d.~sample $S$ of size $n$ from an unknown distribution $D$ over a finite space $X$.  The mechanism $\M$ must answer queries from a stateful adversary $\A$.  These queries are adaptive:  $\A$ outputs a query $q_i$, to which the mechanism returns a response $a_i$, and the outputs of $\A$ and $\M$ may depend on all queries $q_1,\ldots,q_{i-1}$ and responses $a_1,\ldots,a_{i-1}$.

\subsection{Statistical Queries and Optimization Queries}

In this work, the first type of query we consider is a \emph{statistical query}, which is specified by a function $q:X\rightarrow[0,1]$ that represents a real-valued statistic for any element $x\in X$.  The restriction of $q$ to $[0,1]$ is for convenience; our results easily generalize to the case where $q$ is merely bounded.  We then define the query $q$ on a sample $S\in X^m$ as $q(S)=\frac{1}{|S|}\sum_{x\in S}{q(x)}$ and on the distribution as $q(D)=\E_{x\sim D}[q(x)]$.  This represents the average value of the statistic on the sample and distribution, respectively.  We now define the accuracy of $\M$:

%\begin{definition}
%A mechanism $\M$ is said to be $(\alpha,\beta)$\emph{-accurate} over a sample $S$ on low-sensitivity queries $q_1,\ldots,q_k$ if for its responses $a_1,\ldots,a_k$ we have
%\[\Pr_{\M,\A}\left[\max_i |q_i(S) - a_i| \le \alpha\right] \ge 1-\beta.\]
%\end{definition}
%
%%The key requirement is stronger.  Namely, 
%We also seek accuracy over the unknown distribution.

\begin{definition}\label{def:lsq_accuracy}
A mechanism $\M$ is $(\alpha,\beta)$\emph{-accurate} over distribution $D$ on statistical queries $q_1,\ldots,q_k$, if when $\M$ is given an i.i.d.~sample $S$ from $D$, for its responses $a_1,\ldots,a_k$ we have
\[\Pr_{\M,\A}\left[\max_i |q_i(D) - a_i| \le \alpha\right] \ge 1-\beta.\]  
\end{definition}

We define $(\alpha,\beta)$-accuracy over a sample $S$ analogously.
%Using this notation, we seek out mechanisms that are $(\alpha,\beta)$-accurate on $D$ with access not to $D$ directly but only an i.i.d.~sample $S$ from $D$.  However, we also want mechanisms that are fast.  
In this work, we not only desire $(\alpha,\beta)$-accuracy but we also want to consider the time per query taken by $\M$.  
%In particular, we want $\M$ to take no more than approximately $\log(n)$ time for each query.  
%In general, this will not be possible if merely computing $q(S)$ for a query $q$ on a given sample $S$ of size $\ell$ takes more than $O(\ell)$ time.  
We assume we will have oracle access to $q$, which will compute $q(s)$ for a sample point $s$ in unit time (and also $q(S)$ in at most $O(|S|)$ time).  This is not a strong assumption:  As long as the queries can be computed efficiently, then this can add only at most a poly-log factor overhead in $n$ and $|X|$ (as long as we only compute $q$ on a roughly $\log(n)$ size sample, which will turn out to be exactly the case).

We also consider optimization queries, first considered in this adaptive setting by \citet{bassily2016algorithmic}.  In convex optimization, we have a loss function $\ls:X^n\times \Theta\rightarrow\mathbb{R}$ defined over a convex set $\Theta\subseteq\mathbb{R}^d$ and a sample from $X^n$ drawn from a distribution $D$, and the goal is to output $x\in\Theta$ that minimizes the expected loss, i.e.~such a query is defined as
\[q(D) := \argmin_{x\in\Theta}\E_{S\sim D^n}[\ls(S,x)].\]
Since the loss function $\ls$ determines the query, we will abuse notation and use $\ls$ to also refer to the optimization query.
We measure accuracy of the response $a_i$ by the expected regret:  A mechanism is $(\alpha,\beta)$\emph{-accurate} on optimization queries each specified by a loss function $\ls_i$ with respect to a distribution $D$ if 
\[\Pr_{\M,\A}\left[\max_i \E_{S}\left[\ls_i(S,a_i) - \min_{x\in\Theta}\ls_i(S,x)\right] \le \alpha\right] \ge 1-\beta.\]  

We will assume that $\ls_i$ is convex in $x$.  We will also consider the special case when $\ls_i$ is strongly convex in $x$.  A function $\ls $ is $H$-\emph{strongly convex} if for all $x,y$ in $\Theta$, 
\[\ls(y)\ge \ls(x) + \langle\nabla\ls(x),y-x\rangle + \frac{H}{2}\|y-x\|_2^2.\]

\subsection{Counting Queries and Sampling Counting Queries}
Counting queries ask the question ``What proportion of the data satisfies property $q$?''  Counting queries are a simple and important restriction of statistical queries~\citep{blum2008learning,bun14fingerprinting,steinke2015between} that limits the allowed statistics to binary properties.  More formally, a counting query is specified by a function $q:X\rightarrow\{0,1\}$, where $q(S)=\frac{1}{|S|}\sum_{s\in S}{q(s)}$ and $q(D)=\E_{s\sim D}[q(s)]$.  As in the statistical query setting, an answer to a counting query must be close to $q(D)$ (Definition~\ref{def:lsq_accuracy}).

This means, however, that answers to counting queries will not necessarily be counts themselves, nor meaningful in settings where we require $\ell$ to be small, i.e.~very few samples from the database to answer each query.  To this end, we introduce \emph{sampling counting queries}.  A sampling counting query (SCQ) is again specified by a function $q:X\rightarrow\{0,1\}$, but this time the mechanism $\M$ must return an answer $a\in\{0,1\}$. Given these restricted responses, we want such a mechanism to act like what would happen if $\A$ were to take a single random sample point $s$ from $D$ and evaluate $q(s)$.  We define queries in this way so that they represent the smallest possible amount of information still useful to an analyst.   

Now the average value the mechanism returns (over the coins of the mechanism) should be close to the expected value of $q$. 
More precisely:%, we want the following:

\begin{definition}
A mechanism $\M$ is $(\alpha,\beta)$-accurate on distribution $D$ for $k$ sampling counting queries $q_i$ if for all states of $\M$ and $\A$, when $\M$ is given an i.i.d.~sample $S$ from $D$,
\[\Pr_{S,\M,\A}\left[\max_i \left|\E_{\M}[\M(q_i)]-q_i(D)\right|\le \alpha\right] \ge 1-\beta.\] 
\end{definition}

We also define $(\alpha,\beta)$-accuracy on a sample $S$ from $D$ analogously.  Again, our requirement is that $\M$ be $(\alpha,\beta)$-accurate with respect to the unknown distribution $D$, this time using only around $\log(n)$ time per query (and a constant number of samples per query).

These queries allow responses to use fewer than the number of sampled points $\ell$ required for answering statistical queries while also returning an integer-valued response.  And if we do want to answer a statistical query, they also allow an anytime algorithm as we can trade off $\ell$ with the accuracy by averaging over the responses to repeatedly asking the same SCQ.  In this way, SCQ's are `honest,' because repeatedly asking the same SCQ always yields an integer fraction instead of a real number.
%\begin{defn}
%A mechanism $\M$ for sampling counting queries is $(\alpha,\beta)$-accurate on a sample $S$ for $k$ sampling statistical queries $q_i$ if for all states of $\M$ and $\A$,
%\[\mathbb{P}_{\M,\A}\left[\max_i \left|E_{\M}[\M(q_i)]-q_i(S)\right|\le \alpha\right] \ge 1-\beta.\] 
%\end{defn}

%\subsection{Counting and Low-Sensitivity Queries}
%
%Now, we would like to return answers for counting queries $q:X\rightarrow\{0,1\}$, but unlike for SCQ's, we don't need to return a $\{0,1\}$ answer.  Rather, we just want to get close to $q(D)$.  The catch is that we are only allowed $\ell\log(n)$ time per query, corresponding to sampling $\ell$ points i.i.d.~from a set of size $n$.
%
%In this section, we use the following notion of accuracy:
%
%\begin{defn}
%A mechanism $\M$ is $(\alpha,\beta)$-accurate on distribution $D$ for $k$ statistical queries $q_i$ if for all states of $\M$ and $\A$, when $\M$ is given an i.i.d.~sample $S$ from $D$,
%\[\mathbb{P}_{S,\M,\A}\left[\max_i \left|\M(q_i)-q_i(D)\right|\le \alpha\right] \ge 1-\beta.\] 
%\end{defn}

\subsection{Differential Privacy and the Transfer Theorem}
Differential privacy, first introduced by \citet{dwork2006calibrating}, provides a strong notion of stability.

\begin{definition}[Differential privacy]
Let $\M$ be a randomized algorithm with domain $X^n$ and image $Z$.  We call $\M$ $(\epsilon,\delta)$\emph{-differentially private} if for every two samples $S,S'\in X^n$ differing on one instance, and every measurable $z\subset Z$, 
\[\Pr[\M(S)\in z] \le e^\epsilon\cdot\Pr[\M(S')\in z] + \delta.\]
  If $\M$ is $(\epsilon,0)$-private, we may simply call it $\epsilon$-private.
\end{definition}

Differential privacy comes with several guarantees useful for developing new mechanisms.
In this paper, we use two well-established differentially-private mechanisms: the Laplace and exponential mechanisms. See \citep{dwork2014book} for more on these mechanisms and properties of differential privacy, including adaptive composition and post-processing, which are also given in Appendix~\ref{sec:dp_appendix} for convenience.

%\subsection{The Transfer Theorem}
A key method of \cite{bassily2016algorithmic} for answering queries adaptively is a `transfer theorem,' which states that if a mechanism is both accurate on a sample and differentially private, then it will be accurate on the sample's generating distribution.  %For our purposes, we may state their result as the following:  %More formally, they show the following:

\begin{theorem}[\citealp{bassily2016algorithmic}]\label{thm:sq_transfer}
Let $\M$ be a mechanism that on input sample $S\sim D^n$ answers $k$ adaptively chosen statistical queries, is $(\frac{\alpha}{64},\frac{\alpha\beta}{32})$-private for some $\alpha,\beta > 0$ and $(\frac{\alpha}{8},\frac{\alpha\beta}{16})$-accurate on $S$.  Then $\M$ is $(\alpha,\beta)$-accurate on $D$.
\end{theorem}

Their `monitoring algorithm' proof technique involves a thought experiment in which an algorithm, called the monitor, assesses how accurately an input mechanism replies to an adversary, and remembers the query it performs the worst on.  It repeats this process some $T$ times, and outputs the query that the mechanism does the worst on over all $T$ rounds.  Since the mechanism is private, so too is the monitor; and since privacy implies stability, this will ensure that the accuracy of the worst query is not too bad.  For more details see \citet{bassily2016algorithmic}.

\section{Answering Statistical Queries}\label{sec:sampling}
In this section, we provide simple and fast mechanisms for answering statistical queries.  We then show that this mechanism is as fast as possible up to poly-log factors when the mechanism gets to ask for evaluations of queries at given sample points.
%\subsection{Fast Mechanisms}
Our mechanism $\M$ for answering statistical queries is as follows:  Given a data set $S$ of size $n$ and query $q$, sample some $\ell$ points uniformly at random from $S$ (with or without replacement), and call this new set $S_\ell$.  Then the mechanism returns $q(S_\ell) + \text{Lap}\left(\frac{1}{\ell\epsilon'}\right)$, where Lap$(b)$ refers to the zero-mean Laplacian distribution with scale parameter $b$, and $\epsilon'$ is a carefully chosen privacy setting.
\begin{algorithm}
	\caption{Fast mechanism for statistical queries}
	\label{alg:sq_mechanism}
\begin{algorithmic}
	\ENSURE Sub-sample size $\ell$, target privacy parameters $(\epsilon,\delta)$, number of queries $k$
	\REQUIRE Sample $S$, query $q$
	\STATE $S_\ell := \{s_1,\ldots,s_\ell\}$, where $s_i\sim S$ uniformly at random (with or without replacement).
	\STATE $\epsilon' := \frac{\epsilon n}{4\ell\sqrt{2k\log(1/\delta)}}$
	\RETURN $q(S_\ell)+\text{Lap}\left(\frac{1}{\ell\epsilon'}\right)$.
\end{algorithmic}
\end{algorithm}
%We start with a mechanism that first samples without replacement.  
%While this technique is fast and can answer any low-sensitivity query, in practice, sampling without replacement is often more expensive than sampling with replacement.

%, which is often preferred in practice.  So in the second half of this section, we then turn to analyzing sampling with replacement.
%guaranteeing a fast mechanism.  Unfortunately, we are only able to show that this is provably private in the case when the queries are counting queries.  So to answer the full range of low-sensitivity queries, we turn to sampling without replacement instead.  We start with our analysis of sampling with replacement.

We may now state our result for mechanism $\M$ (Algorithm~\ref{alg:sq_mechanism}), using suitable values for $\epsilon$, $\delta$, and $\ell$. 

\begin{theorem}\label{thm:lsq} For any $\alpha,\beta>0$ and $k\ge 1$, when we run $\M$ (Algorithm~\ref{alg:sq_mechanism}) on $k$ statistical queries with parameters $\ell = \frac{2\log(4k/\beta)}{\alpha^2}$, $\epsilon=\alpha/64$, and $\delta=\alpha\beta/32$, we have
\leavevmode
\begin{enumerate}
\item $\M$ takes $\tilde{O}\left(\frac{\log(k)\log(k/\beta)}{\alpha^2}\right)$ time per query.
\item $\M$ is $(\alpha,\beta)$-accurate on the distribution so long as
$n = \Omega\left(\frac{\sqrt{k}\log k\cdot\log^{3/2}\left(\frac{1}{\alpha\beta}\right)}{\alpha^2}\right).$
\end{enumerate}
\end{theorem}

Sampling with replacement takes $O(\log n)$ time per sample, for a total of $O(\ell\log n)$ time over $\ell$ samples.  This suffices to prove part 1.\ for the values of $\ell$ and $n$ given.  Sampling without replacement may also take $O(\log n)$ time per sample.\footnote{This may come at the cost of space complexity, e.g. by keeping track of which elements have not been chosen so far~\citep{wong1980sampling}.  Alternatively, there are methods that enjoy optimal space complexity at the cost of worst-case running times, as in rejection sampling~\citep{vitter1984faster}.}

%So Theorem~\ref{thm:lsq} follows after proving part two.  
To prove part 2, we need to take advantage of the fact that sampling amplifies privacy.  If sampling before an $\epsilon$-private mechanism were to only deliver $O(\epsilon)$ instead of $O(\frac{\ell}{n}\epsilon)$ privacy then we would need $\ell > \frac{2\sqrt{2k\log(1/\delta)}\log(2k/\beta)}{\alpha\epsilon}$, which would be undesirable: $\ell$ then becomes the size of the entire database and sampling yields no time savings over computing $q(S)$ exactly.  With these savings in our privacy budget, we can decrease the amount of noise we add to the outputs, compensating for the accuracy loss we incur by sampling.

\begin{proposition}[\citealp{lin2013benefits, balle2018privacy}]\label{prop:privacy_without_replacement}
Given mechanism $\mathcal{P}:X^\ell \rightarrow Y$, let $\M$ do the following:  Sample uniformly at random \emph{without} replacement $\ell$ points from an input sample $S\in X^n$ of size $n$, and call this set $S_\ell$.  Output $\mathcal{P}(S_\ell)$.  Then if $\mathcal{P}$ is $\epsilon$-private, then $\M$ is $\log(1+\frac{\ell}{n}\left(e^{\epsilon}-1\right))$-private for $\ell\ge 1$.
\end{proposition}

Sampling with replacement also amplifies privacy:
\begin{proposition}[\citealp{bun2015subsampling, balle2018privacy}]\label{prop:sampling}
Given mechanism $\mathcal{P}:X^\ell \rightarrow Y$, let $\M$ do the following:  Sample uniformly at random \emph{with} replacement $\ell$ points from an input sample $S\in X^n$, and call this set $S_\ell$.  Output $\mathcal{P}(S_\ell)$.  Then if $\mathcal{P}$ is $\epsilon$-private, then $\M$ is $\log(1+(1-(1-\frac{1}{n})^\ell)(e^\epsilon-1))$-private for $\ell\ge 1$.
\end{proposition}

Note we have that whenever $\epsilon \le 1$, both $\log(1+\frac{\ell}{n}\left(e^{\epsilon}-1\right)) \le 2\frac{\ell}{n}\epsilon$ and $\log(1+(1-(1-\frac{1}{n})^\ell)(e^\epsilon-1)) \le 2\frac{\ell}{n}\epsilon$, so privacy amplification is linear in the sub-sample size $\ell$.  This linear amplification allows us to set $\epsilon'$ as proportional to $\frac{\epsilon n}{\ell\sqrt{k\log(1/\delta)}}$ instead of $\frac{\epsilon}{\sqrt{k\log(1/\delta)}}$ which would be required without any privacy amplification.
For the proof of part 2, see Appendix~\ref{sec:answering_sqs_appendix}. 

We also have a version of this theorem that demonstrates that this mechanism will still be accurate in expectation at any point along the execution, even if in the first $t$ rounds it (with small probability) failed to be accurate.  This requires a slight variant of Theorem~\ref{thm:sq_transfer}, provided in Appendix~\ref{sec:transfer_exp}.
\begin{theorem}\label{cor:lsq_exp}
For any $\alpha \ge \alpha_0 = \tilde{O}\left(\frac{k^{1/4}}{\sqrt{n}} + \frac{1}{\sqrt{\ell}}\right)$, when we run $\M$ (Algorithm~\ref{alg:sq_mechanism}) with parameters $\ell \ge 1$, $\epsilon =\alpha/8$, and $\delta=\alpha/4$, with respect to any possible simulation between $\A$ and $\M$ up to the first $t-1$ rounds, and denoting the expectation while conditioning on any such possibility $E_{t-1}[\cdot]$, for any $i\ge t$,
\[\E_{t-1,S,\A,\M}[|a_i - q_i(D)|] \le \alpha.\]
%when
%$n \ge \frac{\sqrt{k\log(1/\alpha)}}{\alpha^2}$ and $\ell > \frac{1}{\alpha^2}$.
\end{theorem}
The proof is similar to the proof of Theorem~\ref{thm:lsq}, but using Proposition~\ref{prop:transfer_in_exp} and the fact that
\[\E_{t-1,S,\A,\M}[|a_i - q_i(S)|] \le \E_{t-1,S,\A,\M}[|a_i - q_i(S_\ell)|] + \E_{t-1,S,\A,\M}[|q_i(S_\ell) - q_i(S)|] \lesssim \frac{1}{\epsilon' \ell} +\frac{1}{\sqrt{\ell}}.\]

%\subsection{Lower Bound}
The computational complexity of the mechanism in Theorem~\ref{thm:lsq} is tight up to poly-log factors, even in the non-adaptive case when all queries must be made before seeing any replies from the mechanism.  We show this by considering random queries, which for the purposes of this construction, the learner can access by asking for 
evaluations at given points.
The query values will simulate flipping a coin with given bias from one of two biases randomly selected.  Then it takes computing each query on $\Omega(1/\alpha^2)$ sample points to distinguish between a fair coin and a weighted coin, resulting in a total computational complexity of at least $\Omega(k/\alpha^2)$ points.  The proof may be found in Appendix~\ref{sec:lower_bound}.

\begin{proposition}\label{prop:lower_bound}
Suppose for any sequence $q_1,\ldots,q_k$ of $k$ statistical queries chosen non-adaptively, there is a mechanism $\M$ that is $(\alpha,1/5)$-accurate on the uniform distribution over a universe $X$ with $|X| \ge \frac{2\log(10)}{\alpha^2}$.  Then $\M$ must evaluate the queries on at least $\Omega(k/\alpha^2)$ points.
\end{proposition}
%\begin{proof}
%Consider a distribution $Q$ over statistical queries defined by the following process:  For each $i\in[k]$, let $p_i = 1/2$ independently with probability $1/2$ and $p_i=1/2+4\alpha$ with probability $1/2$.  Then set $q_i(x)=1$ with probability $p_i$ independently for each $x\in X$.  Now suppose $\M$ is $(\alpha,4/5)$-accurate on the uniform distribution $U$.  Consider the $i\in[k]$ for which $\M$ computed the value of $q_i$ on the fewest number of samples, and denote these values $q_i(s_1),\ldots, q_i(s_m)$.  Since $\M$ is $(\alpha,4/5)$-accurate for any set of $k$ statistical queries, in particular it remains that accurate for a random set of $k$ queries drawn from $Q$:
%\[P_{Q,\M}[|q_i(U) - a_i| > \alpha] \le 1/5.\]
%From the Hoeffding bound and our assumption on the size of $|X|$, we also have
%\[P_{Q}[|q_i(U) - p_i| > \alpha/2] \le 2e^{-|X|\alpha^2/2} \le 1/5.\]
%
%Thus with probability at least $3/5$, $|a_i-p_i| \le 3\alpha/2$.  Now define a mechanism $A(q_i(s_1),\ldots,q_i(s_m)) = 1/2$ if $a_i \le 1/2+2\alpha$ and otherwise $A(q_i(s_1),\ldots,q_i(s_m)) = 1/2+4\alpha$.  Recall $q_i(s_1),\ldots,q_i(s_m)$ are i.i.d.~draws from a coin with bias either $1/2$ or $1/2+4\alpha$.  Thus with probability at least $3/5$, $A$ distinguishes between the two coins.  This is well known to require $m \ge \Omega(1/\alpha^2)$ (e.g.~see~\citet{bar2002complexity}), which in turn implies that $\M$ computed the value of queries at least $\Omega(k/\alpha^2)$ times.
%\end{proof}

\section{Answering Sampling Counting Queries}\label{sec:scqs}

We now turn to sampling counting queries.  Because of the different notion of accuracy for these queries, we establish a new transfer theorem.  
\begin{theorem}\label{thm:transfer}
Let $\M$ be a mechanism that on input sample $S\sim D^n$ answers $k$ adaptively chosen sampling counting queries, is $(\frac{\alpha}{64},\frac{\alpha\beta}{16})$-private for some $\alpha,\beta > 0$ and $(\alpha/2,0)$-accurate on $S$.  Suppose further that $n\ge \frac{1024\log(k/\beta)}{\alpha^2}$.  Then $\M$ is $(\alpha,\beta)$-accurate on $D$.
\end{theorem}

This allows us to answer sampling counting queries:
\begin{theorem}\label{thm:sampling_thm}
For any $\alpha,\beta >0$ and $k\ge 1$, there is a mechanism $\M$ that satisfies the following: 
\leavevmode
\begin{enumerate}
\item $\M$ takes $\tilde{O}\left(\log\left({\frac{k \log(\frac{1}{\beta})}{\alpha}}\right)\right)$ time per query.
\item $\M$ is $(\alpha,\beta)$-accurate on $k$ SCQ's, where $n \ge \Omega\left(\max\left({{\sqrt{k \log(\frac{1}{\alpha\beta})}}/{\alpha^2}},{\log(k/\beta)}/{\alpha^2}\right)\right).$
%\begin{align*}
%n &\ge \Omega\left(\max\left({\frac{\sqrt{k \log(\frac{1}{\alpha\beta})}}{\alpha^2}},\frac{\log(k/\beta)}{\alpha^2}\right)\right).\\
%\end{align*}
\end{enumerate}
\end{theorem}

This results in spending $\tilde{O}\left(k\log\left({\frac{k \log(\frac{1}{\beta})}{\alpha}}\right)\right)$ time over the course of $k$ queries, which must be tight up to log factors, as the mechanism of course must spend at least unit time per query.

We prove our transfer theorem using the monitoring algorithm $\mathcal{W}_D$ (Algorithm~\ref{alg:monitor_exp_mech}), which takes as input $T$ sample sets, and outputs a query with probability proportional to how far away the query is on the sample as opposed to the distribution.

%\begin{defn}[Exact monitor]
%Define a monitoring algorithm $\mathcal{\hat{W}}_D$ as the following:
%Given input $\bold{S}=\{S_1,\ldots,S_T\}$, for each of $t\in[T]$, simulate
%$\M(S_t)$ and $\A$ interacting, and let $q_{t,1},\ldots,q_{t,k}$ be the queries of $\A$. 
%%and $a_{t,1},a_{t_k}$ be the responses of $\M$.  
%Let \[t^*,i^* = \argmax_{t,i} |q_{t,i}(S_t)-q_{t,i}(D)|.\]  Return $(q_{t^*,i^*},t^*)$.
%\end{defn}

\begin{algorithm}
	\caption{Monitor with exponential mechanism $\mathcal{W}_D$}
	\label{alg:monitor_exp_mech}
\begin{algorithmic}
	\ENSURE Mechanisms $\M$ and $\A$, distribution $D$
	\REQUIRE Set of samples $\bold{S}=\{S_1,\ldots,S_T\}$
	\FOR{$t$ in $[T]$}
	\STATE Simulate $\M(S_t)$ and $\A$ interacting.
	\STATE Let $q_{t,1},\ldots,q_{t,k}$ be the queries of $\A$.
	\ENDFOR
	\STATE Let $\mathcal{R}:=\{(q_{t,i},t)\}_{t\in[T],i\in[k]}$.
	\STATE Abusing notation, for each $t$ and $i\in[k]$, consider the corresponding element $r_{t,i}$ of $\mathcal{R}$ and define the utility of $r_{t,i}$ as $u(\bold{S},r_{t,i}) = |q_{t,i}(S_t)-q_{t,i}(D)|$.
	\RETURN $r\in\mathcal{R}$ with probability proportional to $\exp\left(\frac{\epsilon\cdot n\cdot u(\bold{S},r)}{2}\right)$.
\end{algorithmic}
\end{algorithm}

%\begin{definition}[Monitor with exponential mechanism]\label{def:monitor_exp_mech}
%Define a monitoring algorithm $\mathcal{W}_D$ as the following:
%Given input $\bold{S}=\{S_1,\ldots,S_T\}$, for each of $t\in[T]$, simulate
%$\M(S_t)$ and $\A$ interacting, and let $q_{t,1},\ldots,q_{t,k}$ be the queries of $\A$.
%
%Let $\mathcal{R}=\{(q_{t,i},t)\}_{t\in[T],i\in[k]}$.  Abusing notation, for each $t$ and $i\in[k]$, consider the corresponding element $r_{t,i}$ of $\mathcal{R}$ and define the utility of $r_{t,i}$ as $u(\bold{S},r_{t,i}) = |q_{t,i}(S_t)-q_{t,i}(D)|$.  Release $r\in\mathcal{R}$ with probability proportional to $\exp\left(\frac{\epsilon\cdot n\cdot u(\bold{S},r)}{2}\right)$.
%\end{definition}

%\begin{proof}
%A single pertubation to $\mathbf{S}$ can only change one $S_t$, for some $t$.  Then since $\M$ on $S_t$ is $(\epsilon,\delta)$-private, $\M$ remains $(\epsilon,\delta)$-private over the course of the $T$ simulations.  Since $\A$ uses only the outputs of $\M$, $\A$ is just post-processing $\M$, and therefore it is $(\epsilon,\delta)$-private as well:  releasing all of $\mathcal{R}$ remains $(\epsilon,\delta)$-private.
%
%Since the sensitivity of $u$ is $\Delta=1/n$, the monitor is just using the exponential mechanism to release some $r\in\mathcal{R}$, which is $\epsilon$-private.  Using the standard composition theorem finishes the proof.
%\end{proof}

$\mathcal{W}_D$ must be private if $\M$ is:  $\mathcal{R}$ represents post-processing from the differentially private $\M$, and outputting an element from $\mathcal{R}$ is achieved with the exponential mechanism.
We can then bound the probability that $q(S)$ is far from $q(D)$ for $q$ the query that the monitor outputs, by using the fact that private algorithms like the monitor are also stable.  This yields the transfer theorem given in Theorem~\ref{thm:sampling_thm}.  The full proof is provided in Appendix~\ref{sec:scq_transfer}.

With a transfer theorem in hand, we now introduce a private mechanism for answering SCQ's.

\begin{algorithm}
	\caption{SCQ mechanism}
	\label{alg:scq_mech}
\begin{algorithmic}
	\ENSURE Target accuracy $\alpha$
	\REQUIRE Sample $S$, query $q$
	\STATE Sample $s\sim S$ uniformly at random.
	\RETURN $q(s)$ with probability $1-\alpha$ and $1-q(s)$ with probability $\alpha$.
\end{algorithmic}
\end{algorithm}

\begin{lemma}[SCQ mechanism]\label{lem:sampling}
For $\epsilon\le 1$, There is an $(\epsilon,\delta)$-private mechanism to release $k$ SCQ's that is $(\alpha,0)$-accurate, for $\alpha\le1/2$, with respect to a fixed sample $S$ of size $n$ so long as
$n > \frac{2\sqrt{2k\log(1/\delta)}}{\alpha\epsilon}.$
\end{lemma}
\begin{proof}
We design a mechanism $\M$ to release an $(\alpha,0)$-accurate SCQ for $n>\frac{1}{\alpha\epsilon}$ and then use adaptive composition.  The mechanism (Algorithm~\ref{alg:scq_mech}) is simple:  sample $s$ i.i.d.\ from $S$.  Then release $q(s)$ with probability $1-\alpha$ and $1-q(s)$ with probability $\alpha$.  Let $i=\sum_{s\in S} q(s)$.  Then $\E_{\M}[\M(q)] = \frac{(1-\alpha)i + \alpha(n-i)}{n} = \frac{i}{n} +\alpha\left(\frac{n-2i}{n}\right)$, so $\frac{i}{n}-\alpha \le \E_{\M}[\M(q)] \le \frac{i}{n}+\alpha$, implying that $\M$ is $(\alpha,0)$-accurate on $S$.

Now let $S'$ differ from $S$ on one element $s$, where $q(s)=0$ but for $s'\in S'$, $q(s')=1$.  The other cases are very similar.  Consider
\[\frac{\Pr[\M(S') = 1]}{\Pr[\M(S)=1]} = \frac{(1-\alpha)\frac{i+1}{n} +\alpha\left(\frac{n-(i+1)}{n}\right)}{(1-\alpha)\frac{i}{n} +\alpha\left(\frac{n-i}{n}\right)} = 1 + \frac{1-2\alpha}{(1-2\alpha)i+\alpha n}.\]
Note this is at least $1$ since $1-2\alpha\ge 0$.  By computing the partial derivative with respect to $i$, it is easy to see that this is maximized when $i=0$ or $i=n-1$.  When $i=0$, 
\[\log\left(\frac{\Pr[\M(S') = 1]}{\Pr[\M(S)=1]}\right) \le \frac{1-2\alpha}{\alpha n} \le \frac{1}{\alpha n} \le \epsilon\] when $n\ge \frac{1}{\epsilon \alpha}$.  When $i=n-1$,
\[\log\left(\frac{\Pr[\M(S') = 1]}{\Pr[\M(S)=1]}\right) \le \frac{1-2\alpha}{n(1-\alpha)-(1-2\alpha)} \le \epsilon\] when $n\ge \frac{(1-2\alpha)(\epsilon+1)}{(1-\alpha)\epsilon}$ but because $\frac{1-2\alpha}{1-\alpha}\le 1$, it suffices to set $n\ge 1+\frac{1}{\epsilon}$.  The proof is completed by noting that $\frac{1}{\epsilon\alpha}\ge 1+\frac{1}{\epsilon}$ because $\epsilon\le 1$.
\end{proof}

%\begin{lemma}[Naive $\epsilon$-private algorithm for SSQs]
%There is an $(\epsilon,\delta)$-private mechanism to release k SSQs that is $(\alpha,\beta)$-accurate with respect to a fixed sample $S$ of size $n$ so long as \[n>\frac{2\sqrt{2k\log(1/\delta)}\log(1/\beta)}{\alpha\epsilon}.\]
%\end{lemma}
%\begin{proof}
%We first give an $\epsilon'$-private mechanism $\mathcal{M}$ for releasing $E_S[q]$.  Then our mechanism $\mathcal{M'}$ will return $1$ with probability $\mathcal{M}(q)$ and $0$ otherwise.  (When $\mathcal{M}(q) <0$ or $\mathcal{M}(q)>1$, we just round to $0$ and $1$ respectively).  $\mathcal{M}'$ is post-processing of $\mathcal{M}$, so it will be $\epsilon'$-private if $\mathcal{M}$ is.  Releasing $E_S[q]$ is merely answering a counting query, so adding Laplacian noise $Lap(\Delta/\epsilon')$ to the true value $E_S[q]$, for $\Delta=1/n$ the global sensitivity, suffices for $\epsilon'$-privacy.  Laplacian noise is well-behaved: we have $P[Lap(b)\ge \alpha] = e^{-\alpha/b}$.  Then $P[|E_\mathcal{M'}[\mathcal{M'}(q)] - E_S[q]| > \alpha] = P[|Lap(\frac{1}{n\epsilon'})| > \alpha] < e^{-\alpha n\epsilon'} = \beta$ when $n>\frac{\log(1/\beta)}{\alpha\epsilon'}$.  This is to answer one such query.  To answer $k$ of them, we use Proposition~\ref{prop:composition}, which implies we need $n>\frac{2\sqrt{2k\log(1/\delta)}\log(1/\beta)}{\alpha\epsilon}$.
%\end{proof}

We now use this mechanism to answer sampling counting queries.
\begin{proofof}{Theorem~\ref{thm:sampling_thm}}
We use Algorithm~\ref{alg:scq_mech} for each query.  This gives an $(\epsilon,\delta)$-private mechanism that is $(\alpha/2,0)$-accurate so long as
$n \ge \frac{4\sqrt{2k\log(1/\delta)}}{\alpha\epsilon}$. 
Setting $\epsilon$ and $\delta$ as required by Theorem~\ref{thm:transfer} implies that we need 
$n \ge \Omega\left(\sqrt{k \log(\frac{1}{\alpha\beta})}/\alpha^2\right)$. 
Note to use Theorem~\ref{thm:transfer} we also need $n\ge \Omega\left(\log(k/\beta)/\alpha^2\right)$.
The sample complexity bound follows.  This mechanism samples a single random point, taking $O(\log(n))$ time, completing the proof.
\end{proofof}

%Note that it is possible to simulate a mechanism for answering counting queries with our mechanism for sampling counting queries:  The average of some $\ell$ sampling counting queries is a response to a counting query.  This enforces `honesty' for counting queries as well, since the returned value is an actual count: it is always an integer fraction of $\ell$, instead of an arbitrary real number due to added noise.  See Appendix~\ref{sec:comparison} for more details.

\section{Comparing Counting and Sampling Counting Queries}\label{sec:comparison}

How do our mechanisms for counting queries and sampling counting queries compare? Can we use a mechanism for SCQ's to simulate a mechanism for counting queries, or vice-versa?  We now show that the natural approach to simulate a counting query with SCQ's results in similar running times but an extra $O(1/\alpha)$ factor in its sample size (although it does enjoy a slightly better dependence on $k$).  This represents a $O(1/\alpha)$ overhead to enforce `honesty' for counting queries as well, since the returned value is now an actual count: it is always an integer fraction of $\ell$, instead of an arbitrary real number due to added noise. 

\begin{proposition}
Using $\ell$ SCQ's to estimate each counting query is an $(\alpha,\beta)$-accurate mechanism for $k$ counting queries if $\ell \ge \frac{2\log(4k/\beta)}{\alpha^2}$ and 
$n = \Omega\left(\frac{\sqrt{k\log k}\log^{3/2}(\frac{1}{\alpha\beta})}{\alpha^3}\right)$.
\end{proposition}
\begin{proof}
The mechanism, for each query $q$, will query the SCQ mechanism $\M$ described in Section~\ref{sec:scqs} $\ell$ times with the query $q$, and return the average, call this $a_q$.  Note that $\E[a_q] = \E[\M(q)]$.  Since each SCQ is independent of each other, a Hoeffding bound gives $\Pr[|a_q - \E[a_q]|\ge \alpha/2] \le 2e^{-\ell\alpha^2/2} \le \beta/2k$ when $\ell \ge \frac{2\log(4k/\beta)}{\alpha^2}$.  Using Theorem~\ref{thm:sampling_thm}, as long as
$n =\Omega\left({\frac{\sqrt{k\ell} \log(\frac{1}{\alpha\beta})}{\alpha^2}}\right)$,
we have that $\Pr[\max_q |\E[\M(q)] - q(D)|\ge \alpha/2] \le \beta/2$, over all $k\ell$ queries.  Then the union bound implies that
	\begin{align*}
		\Pr[\max_q |a_q - q(D)|\ge \alpha] &\le \Pr[\max_q |a_q-\E[\M(q)]| +|\E[\M(q)] - q(D)|\ge \alpha] \\
		&\le \beta/2 + \beta/2 \le \beta,
	\end{align*}
%Plugging in $\ell$ into the above expression for $n$ completes the proof.
completing the proof.
\end{proof}

Meanwhile, it is possible to use a mechanism for counting queries to attempt to answer SCQ's, but it has higher sample complexity than the mechanism for SCQ's proposed above. Indeed, there is the na\"ive approach that ignores time constraints by first computing $q(S)$ exactly, adding noise to obtain a value $\tilde{a}_q$, and then returning $1$ with probability $\tilde{a}_q$ and $0$ otherwise.  For this mechanism we obtain an $(\epsilon,\delta)$-private mechanism to release $k$ SCQ's that is $(\alpha,\beta)$-accurate with respect to a fixed sample $S$ of size $n$ so long as $n>\frac{2\sqrt{2k\log(1/\delta)}\log(1/\beta)}{\alpha\epsilon},$ which is strictly worse than the mechanism for SCQ's we actually use.  This motivates our approach to SCQ's.

\section{An Application to Convex Optimization}\label{sec:optimization}

We now show how to use our fast mechanism for statistical queries to get improved responses to adaptive convex optimization queries.
%Our mechanism for convex optimization will take advantage of our fast mechanism for statistical queries.  
To minimize a loss function $\ls$, we will perform gradient descent but we calculate each coordinate $j\in[d]$ of each gradient using Algorithm~\ref{alg:sq_mechanism} via the statistical query $q_{t-1,j}(S) = \nabla\ls(S,x_{t-1})^{(j)}$, as described in Algorithm~\ref{alg:gd}.  The mechanism, recall, draws a random subsample $S_\ell$ and adds independent noise which we'll call $b$, so that it returns $\tilde{\nabla}\ls(S,x_{t-1})^{(j)} := \nabla\ls(S_\ell,x_{t-1})^{(j)}+b_{j,t-1}.$  We may abbreviate $\tilde{\nabla}\ls(S,x_{t})$ as $\tilde{\nabla}\ls(x_{t})$, or $\tilde{\nabla}_t$.  We then repeat Algorithm~\ref{alg:gd} $k$ times, once for each convex optimization query.

To do this, we need to assume the restriction of the gradient to each coordinate $\nabla\ls(S,x)^{(j)}$ is a statistical query.  If this is the case, we call such a gradient \emph{statistical}.  This is not a strong assumption: it is the case when for example the loss is of the form $\ls(S,x) = \frac{1}{|S|}\sum_{s\in S} \ell(s,x)$ for $\ell:X\times\Theta\rightarrow\mathbb{R}$ and $\nabla\ell\in[0,1]$.\footnote{This last requirement may be weakened so that we just require $\nabla\ell$ to be bounded (which happens when $X$ and $\Theta$ are compact, for example).  The stronger requirement for being in $[0,1]$ is because, for convenience, we also required this of statistical queries themselves.}

%Namely, for fixed step size $\eta$ and initial point $x_0$, we define $x_t = x_{t-1}-\eta\tilde{\nabla}\ls(S,x_{t-1})$, where the $i$th coordinate $\tilde{\nabla}\ls(S,x_{t-1})^{(i)}$ is the output of our above mechanism on query 
%We'll use this mechanism to obtain an approximation of the gradient via the query $q_{t-1}(S) := \nabla\ls(S,x_{t-1})^{(j)}$.  
%The mechanism, recall, takes a random subsample $S_\ell$ and adds independent noise which we'll call $b$, so that
%$\tilde{\nabla}\ls(S,x_{t-1})^{(j)} := \nabla\ls(S_\ell,x_{t-1})^{(j)}+b_{j,t-1}.$
%Starting with $x_0$, we repeat this for $T$ steps, and then output the average value $\frac{1}{T}\sum_{t}x_t$.  

\begin{algorithm}
	\caption{Gradient descent with an adaptive mechanism for gradients}\label{alg:gd}
\begin{algorithmic}
	\ENSURE Loss function $\ls$, Mechanism $\M$, learning rate $\eta$
	\REQUIRE number of rounds $T$, initial point $x_0$
	\FOR{$t$ in $[T]$}
	\FOR{$j$ in $[d]$}
	\STATE $q_{t-1,j}(S) := \nabla\ls(S,x_{t-1})^{(j)}$
	\STATE Receive response $a_j := \M(q_{t-1,j},S)$
	\ENDFOR
	\STATE $\tilde{\nabla}\ls(S,x_{t-1}) := (a_1,\ldots,a_d)$%(\M(q_{t-1,j},S))_{\{j\}}$
	\STATE $x_t := x_{t-1}-\eta\tilde{\nabla}\ls(S,x_{t-1})$
	\ENDFOR
	\RETURN $\frac{1}{T}\sum_{t} x_t$.
\end{algorithmic}
\end{algorithm}

We first show that the \emph{expected excess loss} $\E_{S,\mathcal{M},\A}[\ls(S,x) - \min_{x\in\Theta}\ls(S,x)]$ for $x$ the output of Algorithm~\ref{alg:gd} is small for convex functions.  

\begin{theorem}\label{thm:conv_opt}
For each $i\in[k]$, let $\ls_i$ be differentiable and convex, let $\nabla\ls_i$ be statistical, for any $x\in\Theta$, $\E_{S,S'\sim S}[\|\nabla\ls_i(S',x)\|^2] \le G^2$, and finally, for any $x,y\in\Theta$, $\|x-y\|^2 \le D^2$.  Then there is a mechanism that answers $k$ adaptive optimization queries $\ls_i$ each with expected excess loss $\alpha$ if $n =\tilde{O}\left(\frac{d^{3/2}\sqrt{k}}{\alpha^5}\right)$ in a total of $\tilde{O}\left(\frac{dk}{\alpha^2}\right)$ calls to Algorithm~\ref{alg:sq_mechanism} using parameter $\ell = \tilde{O}\left(\frac{d}{\alpha^4}\right)$ and $\tilde{O}\left(\frac{1}{\alpha^2}\right)$ iterations of gradient descent per query.
\end{theorem}

Since $\ls_i$ is convex, we have $\sum_{t=1}^T \E[\ls_i(x_t) - \ls_i(x^*)] \le \sum_{t=1}^T \E[\langle\nabla_t,x_t-x^*\rangle]$, where $x^*=\argmin_{x\in\Theta}\ls_i(x)$.  Then we can bound each term on the right-hand side using the fact that $\E_{t-1}[\nabla_t^{(j)}] \le \E_{t-1}[\tilde{\nabla}_t^{(j)}] +\tilde{O}\left(\frac{R^{1/4}}{\sqrt{n}} + \frac{1}{\sqrt{\ell}}\right)$, i.e.\ Theorem~\ref{cor:lsq_exp}.
Proofs may be found in Appendix~\ref{sec:sco_proofs}.

We can boost this to a high-probability result by running the gradient-descent algorithm $O(\log(k/\beta))$ times and use the exponential mechanism to pick the best run among them, similarly to previous work that does this kind of boosting~\citep{bassily2014private}.  

\begin{corollary}\label{thm:conv_opt_whp}
For each $i\in[k]$, let $\ls_i$ be differentiable and convex, let $\nabla\ls_i$ be statistical, for any $x\in\Theta$, $\E_{S,S'\sim S}[\|\nabla\ls_i(S',x)\|^2] \le G^2$, and finally, for any $x,y\in\Theta$, $\|x-y\|^2 \le D^2$.  Then there is an $(\alpha,\beta)$-accurate mechanism that answers $k$ adaptive optimization queries $\ls_i$ when $n =\tilde{O}\left(\frac{d^{3/2}\sqrt{k}\log(k/\beta)}{\alpha^{5}}\right)$ in a total of $\tilde{O}\left(\frac{dk\log(k/\beta)}{\alpha^2}\right)$ calls to Algorithm~\ref{alg:sq_mechanism} using parameter $\ell = \tilde{O}\left(\frac{d}{\alpha^4}\right)$ and $\tilde{O}\left(\frac{\log(k/\beta)}{\alpha^2}\right)$ iterations of gradient descent per query.
\end{corollary}

We also show an equivalent result holds when the loss function is not only convex but strongly convex (again the proof is in Appendix~\ref{sec:sco_proofs}).
\begin{corollary}\label{thm:str_conv_opt_whp}
For each $i\in[k]$, let $\ls_i$ be differentiable and $H$-strongly convex, let $\nabla\ls_i$ be statistical, and for any $x\in\Theta$, $\E_{S,S'\sim S}[\|\nabla\ls_i(S',x)\|^2] \le G^2$.  Then there is an $(\alpha,\beta)$-accurate mechanism for $k$ adaptive optimization queries $\ls_i$ when $n =\tilde{O}\left(\frac{d^{3/2}\sqrt{k}\log(k/\beta)}{\alpha^{5/2}}\right)$ in a total of $\tilde{O}\left(\frac{dk\log(k/\beta)}{\alpha}\right)$ calls to Algorithm~\ref{alg:sq_mechanism} using parameter $\ell = \tilde{O}\left(\frac{d}{\alpha^2}\right)$ and $\tilde{O}\left(\frac{\log(k/\beta)}{\alpha}\right)$ iterations of gradient descent per query.
\end{corollary}

%\section{Conclusions}
%In this paper, we have introduced faster mechanisms that take advantage of the simultaneous ability of sampling to boost privacy while decreasing running time.  It would be interesting to examine other adaptive settings in which sampling can help as much as it does in this work.  Sub-linear time algorithms are frequently required for a variety of problems, such as property testing or large-data environments.  A future direction of interest is understanding how can fast algorithms for adaptive analysis be developed in these types of settings.

%\subsubsection*{Acknowledgements}
%\clearpage
\acks{%
Benjamin Fish was supported in part by the NSF EAPSI fellowship and NSF grant IIS-1526379. 
Lev Reyzin was supported in part by NSF grants IIS-1526379 and CCF-1848966.
Benjamin Rubinstein acknowledges support of the Australian Research Council (DP150103710).
} 

%\newpage

\bibliography{refs}

%\section*{APPENDIX}

\appendix

\section{Answering Statistical Queries}\label{sec:answering_sqs_appendix}

In this section, we point out how Algorithm~\ref{alg:sq_mechanism} may be used with appropriate sub-sample size and privacy parameters to answer statistical queries both quickly and accurately:
\begin{proofof}{Theorem~\ref{thm:lsq}, part 2}
Since the Laplace mechanism receives a sample $S_\ell$ of size $\ell$, output $a_q$ can be bounded with the standard accuracy result for the Laplace mechanism ensuring $\epsilon'$-privacy:  
$\Pr[|a_q-q(S_\ell)| \ge \alpha/2] \le e^{-\frac{\alpha\epsilon'\ell}{2}}.$
We can bound this above by $\frac{\beta}{2k}$ provided $\epsilon' \ge \frac{\log(2k/\beta)}{\ell\alpha}$.  Recalling that $\epsilon' = \frac{\epsilon n}{4\ell\sqrt{2k\log(1/\delta)}}$, this occurs when
\[n \ge \frac{4\sqrt{2k\log(1/\delta)}\log(2k/\beta)}{\alpha\epsilon}.\]  From the Hoeffding bound, we also get that
$\Pr[|q(S_\ell)-q(S)| \ge \alpha/2] \le 2 e^{-\frac{\alpha^2\ell}{2}}.$
Once again we can bound this above by $\frac{\beta}{2k}$ so long as $\ell \ge \frac{2\log(4k/\beta)}{\alpha^2}$.  

Thus  for all $q$, $\Pr[|a_q-q(S)| \ge \alpha]\le \Pr[|a_q-q(S_\ell)| \ge \alpha/2] + \Pr[|q(S_\ell)+q(S)| \ge \alpha/2] \le \beta/k.$  The union bound immediately yields $(\alpha,\beta)$-accuracy on the sample over all $k$ queries.
From either Proposition~\ref{prop:privacy_without_replacement} or~\ref{prop:sampling}, we also have that on a single query this mechanism is $\left(2\frac{\ell}{n}\epsilon'\right)$-private, where $2\frac{\ell}{n}\epsilon'= \frac{\epsilon}{2\sqrt{2k\log(1/\delta)}}$.  Thus by the adaptive composition lemma (see Appendix~\ref{sec:dp_appendix}), the mechanism over the course of $k$ queries is $\epsilon$-private.  
%Equivalently, we have $\epsilon'$-privacy when $n\ge\frac{2\log(2k/\beta)}{\epsilon'\alpha}$.  Using adaptive composition, we can answer $k$ queries with $(\epsilon,\delta)$-privacy when $\epsilon' = \frac{\epsilon}{2\sqrt{2k\log(1/\delta)}}$, resulting in $(\alpha,\beta)$-accuracy and $(\epsilon,\delta)$-privacy on $S$ so long as $n > \frac{4\sqrt{2k\log(1/\delta)}\log(2k/\beta)}{\alpha\epsilon}$.  
The proof is concluded by applying Theorem~\ref{thm:sq_transfer}.
\end{proofof}

\section{Transfer Theorems}
In this section, we prove our required transfer theorems, which state that if a mechanism is accurate on the sample and private, it will also be accurate on the distribution.

\subsection{Transfer Theorem for Sampling Counting Queries}\label{sec:scq_transfer}
We return to the proof of Theorem~\ref{thm:transfer}, our transfer theorem for sampling counting queries.

First, we show the monitor is private.

%\begin{lemma}\label{lem:private_monitor}
%If $\M$ is $(\epsilon,\delta)$-private for $k$ queries, then $\mathcal{W}_D$ is $(2\epsilon,\delta)$-private.
%\end{lemma}
\begin{lemma}\label{lem:private_monitor}
If $\M$ is $(\epsilon,\delta)$-private for $k$ queries, then $\mathcal{W}_D$ is $(2\epsilon,\delta)$-private.
\end{lemma}
\begin{proof}%[Lemma~\ref{lem:private_monitor}]
A single perturbation to $\mathbf{S}$ can only change one $S_t$, for some $t$.  Then since $\M$ on $S_t$ is $(\epsilon,\delta)$-private, $\M$ remains $(\epsilon,\delta)$-private over the course of the $T$ simulations.  Since $\A$ uses only the outputs of $\M$, $\A$ is just post-processing $\M$, and therefore it is $(\epsilon,\delta)$-private as well:  releasing all of $\mathcal{R}$ remains $(\epsilon,\delta)$-private.

Since the sensitivity of $u$ is $\Delta=1/n$, the monitor is just using the exponential mechanism to release some $r\in\mathcal{R}$, which is $\epsilon$-private.  The standard composition theorem completes the proof.
\end{proof}

We will also need some of the tools used by \citet{bassily2016algorithmic}.
%This transfer theorem only applies to the notion of accuracy for low-sensitivity queries, rather than accuracy for SCQ's.  Consequently, we prove a monitor-based transfer theorem for SCQ's.  We need two of Bassily et al.'s results for our transfer theorem. 
First, for a monitoring algorithm $\mathcal{W}$, the expected value of the outputted query on the sample will be close to its expected value over the distribution---formalizing a connection between privacy and stability.

\begin{lemma}[\citealp{bassily2016algorithmic}]\label{lem:monitor_quality}
Let $\mathcal{W}:(X^n)^T\rightarrow Q\times[T]$ be $(\epsilon,\delta)$-private where $Q$ is the class of statistical queries.  Let $S_i\sim D^n$ for each of $i\in[T]$ and $\mathbf{S}=\{S_1,\ldots,S_T\}$.  Then
\[\left|\E_{\bold{S},\mathcal{W}}[q(D)|(q,t)=\mathcal{W}(\bold{S})] - \E_{\bold{S},\mathcal{W}}[q(S_t)|(q,t)=\mathcal{W}(\mathbf{S})]\right| \le e^{\epsilon}-1+T\delta.\]
%\begin{gather*}
%\left|\E_{\bold{S},\mathcal{W}}[q(D)|(q,t)=\mathcal{W}(\bold{S})]-\E_{\bold{S},\mathcal{W}}[q(S_t)|(q,t)=\mathcal{W}(\mathbf{S})]\right|\\
%\le e^{\epsilon}-1+T\delta.
%\end{gather*}
\end{lemma}

We will also use a convenient form of accuracy bound for the exponential mechanism.
\begin{lemma}[\citealp{bassily2016algorithmic}]\label{lem:exp_mech_accuracy}
	Let $\mathcal{R}$ be a finite set, $f:\mathcal{R}\rightarrow\mathbb{R}$ a function, and $\eta>0$.  Define a random variable $X$ on $\mathcal{R}$ by $\Pr[X=r] = e^{\eta f(r)}/C$, where $C=\sum_{r\in\mathcal{R}} e^{\eta f(r)}$.  Then
$\E[f(X)] \ge\max_{r\in\mathcal{R}} f(r) - \frac{1}{\eta}\log|\mathcal{R}|.$
\end{lemma}

Now we can provide the proof of the transfer theorem:
\begin{proofof}{Theorem~\ref{thm:transfer}}
Consider the results for simulating $T$ times the interaction between $\M$ and $\A$.
Suppose for the sake of contradiction that $\M$ is not $(\alpha,\beta)$-accurate on $D$.  Then for every $i$ in $[k]$ and $t$ in $T$, 
%\[\mathbb{P}_{S_t,\M,\A}\left[\max_i \left|E_{\M}[\M(q_{t,i})]-q_{t,i}(D)\right| > \alpha\right] > \beta.\]
since $|\E_{\M}[\M(q_{t,i})]-q(S_t)| \le \alpha/2$, we have
\[\Pr_{S_t,\M,\A}\left[\max_i \left|q_{t,i}(S_t)-q_{t,i}(D)\right| > \alpha/2\right] > \beta.\]

Call some $q$ and $t$ that achieves the maximum $|q(S_t)-q(D)|$ over the $T$ independent rounds of $\M$ and $\A$ interacting, as $\mathcal{W}_D$ does (Algorithm~\ref{alg:monitor_exp_mech}), by $q_w$ and $t_w$.
%Call $(q^*,t^*)$ the output of $W_D$.  By definition, one of these maxima is achieved by $q^*$.  
Since each round $t$ is independent, the probability that $|q_w(S_{t_w})-q_w(D)| \le \alpha/2$ is then no more than $(1-\beta)^T$.
%, i.e. \[\mathbb{P}_{\bold{S},\mathcal{W}_D}\left[|q_w(S_{t_w})-q_w(D)|>\alpha/2\right] > 1-(1-\beta)^T.\]
Then using Markov's inequality immediately grants us that
\begin{equation}\label{eq:expected_max_is_big}
\E_{\bold{S},\mathcal{W}_D}\left[|q_w(S_{t_w})-q_w(D)|\right] > \frac{\alpha}{2}(1-(1-\beta)^T).
\end{equation}
Let $\Gamma = \E_{\bold{S},\mathcal{W}_D}\left[|q^*(S_{t^*})-q^*(D)| : (q^*,t^*)=\mathcal{W}_D(\bold{S})\right]$.

Setting $f(r) = u(\bold{S},r)$, Lemma~\ref{lem:exp_mech_accuracy} implies that under the exponential mechanism, we have
\[\E[|q^*(S_{t^*})-q^*(D)| : (q^*,t^*)=\mathcal{W}_D(\bold{S})] \ge |q_w(S_{t_w})-q_w(D)| - \frac{2}{\epsilon n}\log(kT).\]
Taking the expected value of both sides with respect to $\bold{S}$ and the randomness of the rest of $\mathcal{W}_D$, we obtain 
\begin{align}
\Gamma &\ge \E_{\bold{S},\mathcal{W}_D}[|q_w(S_{t_w})-q_w(D)|] - \frac{2}{\epsilon n}\log(kT) > \frac{\alpha}{2}(1-(1-\beta)^T)- \frac{2}{\epsilon n}\log(kT),\label{eq:lower_bound}
\end{align}
which follows from employing Equation~\eqref{eq:expected_max_is_big}.
On the other hand, suppose that $\M$ is $(\epsilon,\delta)$-private for some $\epsilon,\delta>0$.  Then by Lemma~\ref{lem:private_monitor}, $\mathcal{W}_D$ is $(2\epsilon,\delta)$-private, and then in turn Lemma~\ref{lem:monitor_quality} implies that 
\begin{equation}\label{eq:upper_bound}
\Gamma \le e^{2\epsilon} -1 +T\delta.
\end{equation}

We will now ensure $\Gamma \ge \alpha/8$, via~\eqref{eq:lower_bound}, and $\Gamma \le \alpha/8$, via \eqref{eq:upper_bound}, yielding a contradiction.
Set $T=\lfloor\frac{1}{\beta}\rfloor$ and $\delta = \frac{\alpha\beta}{16}$.
Then \[e^{2\epsilon} -1 +T\delta \le e^{2\epsilon}-1 + \alpha/16 \le \alpha/8\] 
when $e^{2\epsilon}-1 \le \alpha/16$, which in turn is satisfied when $\epsilon \le \alpha/64$, since $0\le\alpha\le 1$.

On the other side, $1-(1-\beta)^{\lfloor\frac{1}{\beta}\rfloor} \ge 1/2$.  Then it suffices to set $\epsilon$ such that $\frac{2}{\epsilon n}\log(kT) \le \alpha/8$.  Thus we need $\epsilon$ such that
\[\frac{16\log(k/\beta)}{\alpha n}\ \le\ \epsilon\ \le\ \alpha/64.\]  Such an $\epsilon$ exists, since we explicitly required $n\ge \frac{1024\log(k/\beta)}{\alpha^2}$. 
%We obtain a contradiction by setting $\epsilon=\alpha/8$, $\delta=\frac{\alpha\beta}{4}$, and $T=\lfloor\frac{1}{\beta}\rfloor$.
\end{proofof}

\subsection{Transfer Theorem for Statistical Queries in Expectation}\label{sec:transfer_exp}
We also need a transfer theorem for Theorem~\ref{cor:lsq_exp}.  
\begin{proposition}\label{prop:transfer_in_exp}
Consider any possibility for the simulation between $\A$ and $\M$ up to the first $t-1$ rounds.  Denoting the expectation while conditioning on any such possibility $E_{t-1}[\cdot]$, we have for any round $i\ge t$, if $\M$ is $(\alpha/8,\alpha/4)$-private for $\alpha\le 1$, and 
$E_{t-1,S,\M,\A}[|q_i(S) - a_i|] \le \alpha/2,$
%\[E_{t-1,S,\M,\A}[|q_i(S) - a_i|] \le \alpha/2,\]
then
\[E_{t-1,S,\M,\A}[|q_i(D) - a_i|] \le \alpha.\]
\end{proposition}

\begin{algorithm}[h]
	\caption{Monitor $\mathcal{W}$}
	\label{alg:transfer_expec}
\begin{algorithmic}
	\ENSURE Mechanisms $\M$ and $\A$, index $i$, and initial sequence of queries $q_1,\ldots,q_{t-1}$ and responses $a_1,\ldots,a_{t-1}$
	\REQUIRE Sample $S$
	\STATE Set the internal states of $\M(S)$ and $\A$ to be what they would be if the resulting simulation had produced $q_1,\ldots,q_{t-1}$ and $a_1,\ldots,a_{t-1}$.
	\STATE Now simulate $\M(S)$ and $\A$ interacting starting in those states for $i-t+1$ rounds.  Let $q_t,\ldots,q_i$ be the resulting queries. 
	\RETURN $q_i$.
\end{algorithmic}
\end{algorithm}
\begin{proof}
Suppose by way of contradiction that $\E_{t-1,S,\M,\A}[|q_i(D) - a_i|] > \alpha$.
Note the monitor $\mathcal{W}$, given in Algorithm~\ref{alg:transfer_expec}, simply outputs $q_i$, conditioned on $q_1,\ldots,q_{t-1}$ and $a_1,\ldots,a_{t-1}$ being the initial sequence of queries and responses, so
\begin{align*}
	|\E_{S,\mathcal{W}}[q(D)-q(S)|q=\mathcal{W}(S)]| =& |\E_{t-1,S,\M,\A}[q_i(S)-q_i(D)]|\\
	\ge& |\E_{t-1,S,\M,\A}[q_i(D)- a_i]| - |\E_{t-1,S,\M,\A}[q_i(S)- a_i]| \\
	>& \alpha - \alpha/2 = \alpha/2. 
\end{align*}
Since the monitor $\mathcal{W}$ only outputs $q_i$, which is post-processing from a private mechanism $\M$, $\mathcal{W}$ remains $(\alpha/8,\alpha/4)$-private.
Therefore by Lemma~\ref{lem:monitor_quality}, $|\E_{S,\mathcal{W}}[q(D)-q(S)|q=\mathcal{W}(S)]| \le e^\epsilon-1+\delta \le \alpha/2$ with the above values of $\epsilon$ and $\delta$ for $\alpha\le 1$.
\end{proof}

\section{Lower Bound}\label{sec:lower_bound}

\begingroup
\def\thetheorem{\ref{prop:lower_bound}}
\begin{proposition}
Suppose for any sequence $q_1,\ldots,q_k$ of $k$ statistical queries chosen non-adaptively, there is a mechanism $\M$ that is $(\alpha,1/5)$-accurate on the uniform distribution over a universe $X$ with $|X| \ge \frac{2\log(10)}{\alpha^2}$.  Then $\M$ must evaluate the queries on at least $\Omega(k/\alpha^2)$ points.
\end{proposition}
\addtocounter{theorem}{-1}
\endgroup

\begin{proof}
Consider a distribution $Q$ over statistical queries defined by the following process:  For each $i\in[k]$, let $p_i = 1/2$ independently with probability $1/2$ and $p_i=1/2+4\alpha$ with probability $1/2$.  Then set $q_i(x)=1$ with probability $p_i$ independently for each $x\in X$.  Now suppose $\M$ is $(\alpha,1/5)$-accurate on the uniform distribution $U$.  
%Consider the $i\in[k]$ for which $\M$ computed the value of $q_i$ on the fewest number of samples, and denote these values $q_i(s_1),\ldots, q_i(s_m)$.  
Since $\M$ is $(\alpha,1/5)$-accurate for any set of $k$ statistical queries, in particular it remains that accurate for a random set of $k$ queries drawn from $Q$ for any $i$:
\[P_{Q,\M}[|q_i(U) - a_i| > \alpha] \le 1/5.\]
From the Hoeffding bound and our assumption on the size of $|X|$, we also have
\[P_{Q}[|q_i(U) - p_i| > \alpha/2] \le 2e^{-|X|\alpha^2/2} \le 1/5.\]

Thus with probability at least $3/5$, $|a_i-p_i| \le 3\alpha/2$.  For the values $q_i(s_1),\ldots, q_i(s_m)$ that $\M$ computed, define a mechanism $A(q_i(s_1),\ldots,q_i(s_m)) = 1/2$ if $a_i \le 1/2+2\alpha$ and otherwise $A(q_i(s_1),\ldots,q_i(s_m)) = 1/2+4\alpha$.  Recall $q_i(s_1),\ldots,q_i(s_m)$ are i.i.d.~draws from a coin with bias either $1/2$ or $1/2+4\alpha$.  Thus with probability at least $3/5$, $A$ distinguishes between the two coins.  This is well known to require $m \ge \Omega(1/\alpha^2)$ (e.g.~see~\citet{bar2002complexity}), which in turn implies that $\M$ computed the value of queries at least $\Omega(k/\alpha^2)$ times.
\end{proof}

\section{Convex Optimization}\label{sec:sco_proofs}
We now return to the omitted proofs in Section~\ref{sec:optimization}.  Bounding regret here is similar to typical analyses, but is complicated by one major difference:  A typical assumption in stochastic gradient descent is that the oracle returning the oracle for the gradient is unbiased, so that $\E[\tilde{\nabla}\ls] = \nabla\ls$ (e.g.~\citealp{shamir2013stochastic}), whereas here $\E[\tilde{\nabla}\ls]$ is only guaranteed to be close to the true gradient $\ls$.  We take advantage of (strong) convexity to show that for sufficiently large sample size, gradient descent still converges sufficiently quickly.

\begin{theorem}\label{thm:str_conv_opt}
For each $i\in[k]$, let $\ls_i$ be differentiable, $H$-strongly convex, let $\nabla\ls_i$ be statistical, and for any $x\in\Theta$, $\E_{S,S'\sim S}[\|\nabla\ls_i(S',x)\|^2] \le G^2$.  Then there is a mechanism that answers $k$ adaptive optimization queries $\ls_i$ each with expected excess risk $\alpha$ if $n =\tilde{O}\left(\frac{d^{3/2}\sqrt{k}}{\alpha^{5/2}}\right)$ in a total of $\tilde{O}\left(\frac{dk}{\alpha}\right)$ calls to Algorithm~\ref{alg:sq_mechanism} using parameter $\ell = \tilde{O}\left(\frac{d}{\alpha^2}\right)$ and $\tilde{O}\left(\frac{1}{\alpha}\right)$ iterations of gradient descent per query.
\end{theorem}
%[How do these assumptions compare?  Stronger than loss itself being low-sensitivity?  Stronger than the norm being bounded when sampled at one point?]
\begin{proof}%[Theorem~\ref{thm:str_conv_opt}]
We use Algorithm~\ref{alg:gd} to answer $k$ optimization queries, which in turn uses our statistical query oracle (Algorithm~\ref{alg:sq_mechanism}) to get each component of $\nabla\ls_i$, for a total of $R:=k\cdot T\cdot d$ rounds, where $T$ is the number of iterations per optimization.  For each optimization query, we now bound regret.
As is standard, we pick $x^*=\argmin_{x\in\Theta}\ls_i(x)$ to plug in to the definition of strong convexity to get, rearranging,
\[\E[\ls_i(x_t) - \ls_i(x^*)] \le \E[\langle\nabla_t,x_t-x^*\rangle] - \frac{H}{2}\E[\|x_t-x^*\|^2].\]
Again following the standard analysis,
\begin{align*}
	\|x_{t+1}-x^*\|^2 =& \|\Pi(x_t-\eta_t\tilde{\nabla}_t)-x^*\| \le \|x_t - \eta_t\tilde{\nabla}_t - x^*\|^2\\
	\le& \|x_t-x^*\|^2 + \eta_t^2\|\tilde{\nabla}_t\|^2 -2\eta_t\langle\tilde{\nabla}_t,x_t-x^*\rangle.
\end{align*}
In other words,
\[\langle\tilde{\nabla}_t,x_t-x^*\rangle \le \frac{\|x_t-x^*\|^2 - \|x_{t+1}-x^*\|^2}{2\eta_t} + \frac{\eta_t}{2}\|\tilde{\nabla}_t\|^2.\]
Moreover, we can upper-bound $\E[\|\tilde{\nabla}_t\|^2]$ since $\tilde{\nabla}_t = \nabla\ls_i(S_\ell,x_{t})+b_{t}$, where $b_t$ is the noise vector.
\begin{align*}
	\E[\|\tilde{\nabla}_t\|^2] =& \E[\|\nabla\ls_i(S_\ell,x_{t})\|^2] + \E[\|b_t\|^2] + 2\E[\langle \nabla\ls_i(S_\ell,x_{t}),b_t \rangle]\\
	\le& G^2 + 2d\sigma^2 = G^2 +\frac{c dR\log(1/\alpha')}{n^2{\alpha'}^2},
\end{align*}
where $\sigma^2$ is the variance of the noise, $\alpha'$ is the accuracy of Algorithm~\ref{alg:sq_mechanism}, and $c$ is a sufficiently large constant.  Note $\E[\langle \nabla\ls_i(S_\ell,x_{t}),b_t \rangle]=0$ because $b_t$ is independent of both $S_\ell$ and $x_t$.

Now, using the bounds on our oracle, we upper-bound $\langle\nabla_t,x_t-x^*\rangle$ using $\langle\tilde{\nabla}_t,x_t-x^*\rangle$.
% and $\|\tilde{\nabla}_t\|$ using $\|\nabla_t\|$.

Using $\E_{t-1}[\cdot]$ to denote the expectation conditioned on all of the previous $t-1$ iterations, the promise of our mechanism (Theorem~\ref{cor:lsq_exp}) is that we can guarantee that for each coordinate $j$, $\E_{t-1}[\nabla_t^{(j)}] \le \E_{t-1}[\tilde{\nabla}_t^{(j)}] +\alpha'$, where
\[\alpha'=\tilde{O}\left(\frac{R^{1/4}}{\sqrt{n}} + \frac{1}{\sqrt{\ell}}\right).\]
%\[\alpha' = \frac{\sqrt{R\log(1/\delta)}}{n\epsilon}+\frac{1}{\sqrt{\ell}}.\]
Then 
\begin{align*}
	\E[\langle\nabla_t,x_t-x^*\rangle] =& \sum_i \E[\E_{t-1}[\nabla_t^{(j)}(x_t-x^*)^{(j)}]]\\
	\le& \sum_i \E[\E_{t-1}[(\tilde{\nabla}_t^{(j)}+\alpha')(x_t-x^*)^{(j)}]] \\
	=& \E[\langle\tilde{\nabla}_t,x_t-x^*\rangle]+\alpha'\E\left[\sum_i (x_t-x^*)^{(j)}\right]\\
	\le&  \E[\langle\tilde{\nabla}_t,x_t-x^*\rangle]+\alpha'\E[\|x_t-x^*\|_1] \\
	\le& \E[\langle\tilde{\nabla}_t,x_t-x^*\rangle]+\alpha'\sqrt{d}\ \E[\|x_t-x^*\|_2].
\end{align*}
The first equality conditions on the first $t-1$ rounds and then expands the inner product.  The first inequality follows because once we condition on the first $t-1$ rounds, $\nabla_t$ and $x_t$ are independent, so we can use the mechanism's guarantee.  $\tilde\nabla_t$ and $x_t$ are also independent when conditioned on the first $t-1$ rounds, from which the second equality follows.  The last inequality follows from Cauchy-Schwartz.

Note further that $\E[\|x_t-x^*\|_2] \le 1 +  \E[\|x_t-x^*\|_2^2]$, simply because either $\|x_t-x^*\|_2 \le 1$ or $\|x_t-x^*\|_2 < \|x_t-x^*\|_2^2$.  Thus
\[\E[\langle\nabla_t,x_t-x^*\rangle] \le \E[\langle\tilde{\nabla}_t,x_t-x^*\rangle] + \alpha'\sqrt{d} + \alpha'\sqrt{d}\E[\|x_t-x^*\|^2].\]

Thus we have
\begin{align*}
&\sum_{t=1}^T \E[\ls_i(x_t) - \ls_i(x^*)] \\
\le& \sum_{t=1}^T\left(\frac{(1+\alpha'\sqrt{d})E[\|x_t-x^*\|^2] - \E[\|x_{t+1}-x^*\|^2]}{2\eta_t} - \right.\\
&\;\;\;\;\;\;\;\;\; \left. \frac{H}{2}\E[\|x_t-x^*\|^2] +\frac{\eta_t}{2}\left(G^2 +\frac{cdR\log(1/\alpha')}{n^2{\alpha'}^2}\right) + \alpha'\sqrt{d}\right)%\\
\end{align*}

\begin{align*}
	\le& \frac{1}{2}\sum_{t=1}^T \E[\|x_t-x^*\|^2]\left(\frac{1+\alpha'\sqrt{d}}{\eta_t}-\frac{1}{\eta_{t-1}}-H\right) + \left(\frac{G^2}{2} +\frac{cdR\log(1/\alpha')}{2 n^2{\alpha'}^2}\right)\left(\sum_{t=1}^T \eta_t\right) + \alpha'\sqrt{d} T.
\end{align*}

Now if we set $\eta_t = \frac{2}{Ht}$, then $\frac{1+\alpha'\sqrt{d}}{\eta_t}-\frac{1}{\eta_{t-1}}-H \le 0$ when $\alpha'\sqrt{d} \le 1/t$.

Then setting $\alpha'\sqrt{d} \le \frac{1}{T}$, the average loss is
\begin{align*}
	\frac{1}{T}\sum_{t=1}^T \E[\ls_i(x_t) - \ls_i(x^*)] \le& \frac{2}{HT}\left(\frac{G^2}{2} +\frac{cdR\log(1/\alpha')}{2 n^2{\alpha'}^2}\right)\sum_{t=1}^T 1/t + \alpha'\sqrt{d}\\
	\le& \frac{G^2}{H}\cdot\frac{1+\log(T)}{T} + \frac{cdR\log(1/\alpha')}{H n^2{\alpha'}^2}\cdot\frac{1+\log(T)}{T} + \alpha'\sqrt{d}.
\end{align*}

Thus to show that the average loss is no more than $\alpha$ it suffices to show that $\frac{G^2}{H}\cdot\frac{1+\log(T)}{T} \le \alpha/3$, $\alpha'\sqrt{d}\le \alpha/3$, $\frac{cdR\log(1/\alpha')}{H n^2{\alpha'}^2}\cdot\frac{1+\log(T)}{T} \le \alpha/3$, and $\alpha'\sqrt{d}\le 1/T$. For the first, it suffices to set $T = \tilde{O}\left(\frac{G^2}{\alpha H}\right)$.  Then, as long as $\alpha$ is sufficiently small\footnote{That is, $\alpha$ is sufficiently small as a function of $d$, $G$, and $H$.  Or, we can instead assume $G$ and $H$ are absolute constants.  Otherwise, the dependence of $n$ on $G$ and $H$ is messier and we omit these calculations for the sake of brevity.}, it suffices so that $n =\tilde{O}\left(\frac{G^5}{H^{5/2}}\cdot \frac{d^{3/2}\sqrt{k}}{\alpha^{5/2}}\right)$ and $\ell = \tilde{O}\left(\frac{G^4}{H^2}\cdot \frac{d}{\alpha^2}\right)$.  Finally, the number of times we need to compute a gradient over $k$ rounds is $R = k\cdot T\cdot d = \tilde{O}\left(\frac{G^2 kd}{H\alpha}\right)$.
\end{proof}

Corollary~\ref{thm:str_conv_opt_whp} then follows by boosting this to a high-probability result via running the gradient-descent algorithm $\log(k/\beta)$ times and choosing the best run among them using the exponential mechanism. 

%The high probability version then follows:
%\begin{theorem}
%This mechanism is $(\alpha,\beta)$-accurate for $k$ optimization queries, with the same assumptions as in Theorem~\ref{thm:str_conv_opt}, when $n =\tilde{O}\left(\frac{d^{3/2}\sqrt{k}\log(k/\beta)}{\alpha^{5/2}}\right)$ in $\tilde{O}\left(\frac{d^2\log(k/\beta)}{\alpha^3}\right)$ samples per query and $\tilde{O}\left(\frac{\log(k/\beta)}{\alpha}\right)$ iterations of gradient descent per query.
%\end{theorem}

We now turn to the proof of Theorem~\ref{thm:conv_opt}, which is restated here:

\begingroup
\def\thetheorem{\ref{thm:conv_opt}}
\begin{theorem}%\label{thm:conv_opt}
For each $i\in[k]$, let $\ls_i$ be differentiable and convex, let $\nabla\ls_i$ be statistical, for any $x\in\Theta$, $\E_{S,S'\sim S}[\|\nabla\ls_i(S',x)\|^2] \le G^2$, and finally, for any $x,y\in\Theta$, $\|x-y\|^2 \le D^2$.  Then there is a mechanism that answers $k$ adaptive optimization queries $\ls_i$ each with expected excess loss $\alpha$ if $n =\tilde{O}\left(\frac{d^{3/2}\sqrt{k}}{\alpha^5}\right)$ in a total of $\tilde{O}\left(\frac{dk}{\alpha^2}\right)$ calls to Algorithm~\ref{alg:sq_mechanism} using parameter $\ell = \tilde{O}\left(\frac{d}{\alpha^4}\right)$ and $\tilde{O}\left(\frac{1}{\alpha^2}\right)$ iterations of gradient descent per query.
\end{theorem}
\addtocounter{theorem}{-1}
\endgroup

\begin{proof}%[Theorem~\ref{thm:conv_opt}]
The proof is very similar to that of the proof of Theorem~\ref{thm:str_conv_opt}, using the same algorithm, except now we only have
\[\E[\ls(x_t) - \ls(x^*)] \le \E[\langle\nabla_t,x_t-x^*\rangle].\]

But as before, we have
\[\E[\langle\nabla_t,x_t-x^*\rangle] \le \E[\langle\tilde{\nabla}_t,x_t-x^*\rangle] + \alpha'\sqrt{d} + \alpha'\sqrt{d}\ \E[\|x_t-x^*\|^2],\]
\[\langle\tilde{\nabla}_t,x_t-x^*\rangle \le \frac{\|x_t-x^*\|^2 - \|x_{t+1}-x^*\|^2}{2\eta_t} + \frac{\eta_t}{2}\|\tilde{\nabla}_t\|^2,\]
and
\[\E[\|\tilde{\nabla}_t\|^2] \le G^2 +\frac{cdR\log(1/\alpha')}{n^2{\alpha'}^2},\]
for sufficiently large constant $c$.
Then
\begin{align*}
	&\sum_{t=1}^T \E[\ls(x_t) - \ls(x^*)] \\
	\le& \sum_{t=1}^T\left(\frac{(1+\alpha'\sqrt{d})\E[\|x_t-x^*\|^2] - \E[\|x_{t+1}-x^*\|^2]}{2\eta_t} +\frac{\eta_t}{2}\left(G^2 +\frac{cdR\log(1/\alpha')}{n^2{\alpha'}^2}\right) + \alpha'\sqrt{d}\right)\\
	\le& \frac{1}{2}\sum_{t=1}^T \E[\|x_t-x^*\|^2]\left(\frac{1+\alpha'\sqrt{d}}{\eta_t}-\frac{1}{\eta_{t-1}}\right) + \left(\frac{G^2}{2} +\frac{cdR\log(1/\alpha')}{2 n^2{\alpha'}^2}\right)\left(\sum_{t=1}^T \eta_t\right) + \alpha'\sqrt{d}\cdot T\\
	\le& \frac{D^2}{2\eta_T}+\frac{D^2\alpha'\sqrt{d}}{2}\sum_{t=1}^T \frac{1}{\eta_t} + \left(\frac{G^2}{2} +\frac{cdR\log(1/\alpha')}{2 n^2{\alpha'}^2}\right)\left(\sum_{t=1}^T \eta_t\right) + \alpha'\sqrt{d}\cdot T,
\end{align*}
where the last inequality comes from upper-bounding $\|x_t-x^*\|^2$ by the diameter, and collapsing the telescoping series.
Set $\eta_t = \frac{D}{G\sqrt{t}}$.  This gives the average loss as
\begin{align*}
	\frac{1}{T}\sum_{t=1}^T \E[\ls(x_t) - \ls(x^*)] \le& \frac{DG}{2\sqrt{T}}+O\left(\frac{DG\alpha'\sqrt{dT}}{2} + \frac{DG}{2\sqrt{T}} +\frac{D dR\log(1/\alpha')\sqrt{T}}{ G n^2{\alpha'}^2}\right) + \alpha'\sqrt{d}\\
	=& O\left(\frac{DG}{\sqrt{T}}+DG\alpha'\sqrt{dT} + \frac{D dR\log(1/\alpha')\sqrt{T}}{G n^2{\alpha'}^2} + \alpha'\sqrt{d}\right).
\end{align*}

It suffices to show that each of these four terms are upper-bounded by $\alpha/4$, in which case, for sufficiently small $\alpha$,\footnote{As in the proof of Theorem~\ref{thm:str_conv_opt}, this assumption is only required to write $n$ as a function of $D$ and $G$ without having to resort to a much messier formula.  Another alternative is to assume that $D$ and $G$ are absolute constants.} we require $T\ge O(\frac{D^2G^2}{\alpha^2})$, $n\ge\tilde{O}\left(\frac{D^5 G^5 d^{3/2}\sqrt{k}}{\alpha^5}\right)$, and $\ell\ge\tilde{O}\left(\frac{D^4G^4 d}{\alpha^4}\right)$.  Thus the number of times we need to compute a gradient over $k$ rounds is $R = k\cdot T\cdot d = \tilde{O}\left(\frac{D^2G^2 kd}{\alpha^2}\right)$.
\end{proof}

\section{Differential Privacy Review}\label{sec:dp_appendix}
Differential privacy has several nice guarantees, among which is that it composes adaptively.
\begin{lemma}[Adaptive composition; \citealp{dwork2014book,dwork2010boosting}]\label{lemma:composition}
Given parameters $0<\epsilon < 1$ and $\delta > 0$, to ensure $(\epsilon, k\delta' + \delta)$-privacy over $k$ adaptive mechanisms, it suffices that each mechanism is $(\epsilon',\delta')$-private, where
$\epsilon' = \frac{\epsilon}{2\sqrt{2k\log(1/\delta)}}.$
%\[\epsilon' = \frac{\epsilon}{2\sqrt{2k\log(1/\delta)}}.\]
\end{lemma}

We also have a post-processing guarantee:

\begin{lemma}[Post-processing;~\citealp{dwork2014book}]
Let $\M:X^n\rightarrow Z$ be an $(\epsilon,\delta)$-private mechanism and $f:Z\rightarrow Z'$ a (possibly randomized) algorithm.  Then $f\circ\M$ is $(\epsilon,\delta)$-private.
\end{lemma}

\end{document}